\documentclass{article}

\usepackage{microtype}
\usepackage{graphicx}
\usepackage{subcaption}
\usepackage{booktabs} % for professional tables

\usepackage{hyperref}

\usepackage[dvipsnames]{xcolor}

\usepackage[accepted]{icml2023}

\usepackage{amsmath}
\usepackage{amssymb}
\usepackage{mathtools}
\usepackage{amsthm}
\usepackage{thm-restate}

\usepackage[normalem]{ulem}

\usepackage[capitalize,noabbrev]{cleveref}

\theoremstyle{plain}
\newtheorem{theorem}{Theorem}[section]

\newtheorem{lemma}[theorem]{Lemma}
\newtheorem{corollary}[theorem]{Corollary}
\theoremstyle{definition}

\theoremstyle{remark}

\usepackage[textsize=tiny]{todonotes}

\DeclareMathOperator{\im}{im}
\DeclareMathOperator{\tr}{tr}

\DeclareMathOperator{\poly}{poly}
\DeclareMathOperator{\Ric}{Ric}

\newcommand{\Rtot}{R_{\text{tot}}}
\newcommand{\dmin}{d_{\min}}
\newcommand{\dmax}{d_{\max}}
\newcommand{\boundary}{\partial}

\definecolor{ibm-blue}{HTML}{648FFF}
\definecolor{ibm-purple}{HTML}{785EF0}
\definecolor{ibm-magenta}{HTML}{DC267F}
\definecolor{ibm-orange}{HTML}{FE6100}
\newcommand{\first}[1]{{\color{ibm-orange} \mathbf{#1}}}
\newcommand{\second}[1]{{\color{ibm-magenta} \mathbf{#1}}}
\newcommand{\third}[1]{{\color{ibm-purple} \mathbf{#1}}}

\newcommand{\R}{\mathbb{R}}

\definecolor{darkblue}{rgb}{0, 0, 1.0}

\newcommand{\norm}[1]{\left\|#1\right\|}

\icmltitlerunning{Oversquashing and Effective Resistance}

\begin{document}

\twocolumn[
\icmltitle{Understanding Oversquashing in GNNs through the Lens of Effective Resistance}



\icmlsetsymbol{equal}{*}

\begin{icmlauthorlist}
    \icmlauthor{Mitchell Black}{osu}
    \icmlauthor{Zhengchao Wan}{ucsd}
    \icmlauthor{Amir Nayyeri}{osu}
    \icmlauthor{Yusu Wang}{ucsd}
\end{icmlauthorlist}

\icmlaffiliation{osu}{School of Electrical Engineering and Computer Science, Oregon State University, Corvallis, Oregon, USA}
\icmlaffiliation{ucsd}{Hal\i{}c\i{}o\v{g}lu Data Science Institute, University of California San Diego, San Diego, California, USA}

\icmlcorrespondingauthor{Mitchell Black}{blackmit@oregonstate.edu}

\icmlkeywords{graph neural networks, message passing neural networks, oversquashing, effective resistance, total resistance, Laplacian, spectral graph theory, commute time, spectral gap, rewiring}

\vskip 0.3in
]



%\printAffiliationsAndNotice{\icmlEqualContribution}  % leave blank if no need to mention equal contribution
\printAffiliationsAndNotice{} % otherwise use the standard text.

\begin{abstract}
Message passing graph neural networks (GNNs) are a popular learning architectures for graph-structured data. However, one problem GNNs experience is oversquashing, where a GNN has difficulty sending information between distant nodes. Understanding and mitigating oversquashing has recently received significant attention from the research community. In this paper, we continue this line of work by analyzing oversquashing through the lens of the \emph{effective resistance} between nodes in the input graph. Effective resistance intuitively captures the ``strength'' of connection between two nodes by paths in the graph, and has a rich literature spanning many areas of graph theory. We propose to use \emph{total effective resistance} as a bound of the total amount of oversquashing in a graph and provide theoretical justification for its use. We further develop an algorithm to identify edges to be added to an input graph to minimize the total effective resistance, thereby alleviating oversquashing. We provide empirical evidence of the effectiveness of our total effective resistance based rewiring strategies for improving the performance of GNNs.
\end{abstract}

\section{Introduction}

Graph neural networks (GNNs) are powerful tools for graph learning and optimization tasks \cite{scarselli2008graph}. One major framework for GNNs is \textit{message passing}, where node and edge features are repeatedly aggregated locally through node neighborhoods. While it has proven successful, message passing also suffers from several problem related to the topology of the graph. The number of layers of a GNN defines the radius of the neighborhood of a node from which information will be aggregated. When the number of layers is too small, the message passing will only be done locally, and the GNN will not be able to capture information from nodes outside this neighborhood. This problem is known as \textit{underreaching}. On the other hand, choosing a large number of layers can lead to \textit{oversmoothing}, where node features might be smoothed out and become indistinguishable \cite{cai2020note, oono2020Graph}. A third issue is \textit{oversquashing} \cite{alon2021bottleneck}, where as larger neighborhoods are considered, information from long-range interactions passing through certain bottlenecks of the graph will have negligible impact on the training of GNNs. This behaviour was named oversquashing as information from potentially exponentially many (with respect to the number of layers) nodes will be squashed into fixed-sized node vectors.

Understanding when oversquashing occurs is an active area of research. Recently, oversquashing has been analyzed using different techniques such as graph curvature \cite{topping2022oversquashing} and information theory \cite{banerjee2022information}. Moreover, various \textit{rewiring techniques} have been proposed to alleviate oversquashing, where edges are added or removed or edge weights are changed to decrease bottlenecks in the graph before applying GNNs~\cite{arnaiz2022diffwire, deac2022expander, karhadkar2022firstorder, topping2022oversquashing}.

In this paper, we propose to analyze oversquashing through the lens of \emph{effective resistance}. The concept of effective resistance originates from Electrical Engineering~\cite{kirchhoff1847resistance}, where the effective resistance between two nodes $u$ and $v$ in an electrical network is the difference in voltage between $u$ and $v$ when a unit of current is inserted at $u$ and removed at $v$. Since then, effective resistance has taken on a new life in Graph Theory, where effective resistance has been shown to be tied to many properties of the graph underlying the electrical network~\cite{doyle1984random,lyons2017probability}. For example, the effective resistance between a pair of vertices is proportional to the \textit{commute time} between two vertices---the expected number of steps in a random walk from one vertex to the other and back~\cite{chandra1996resistance}. The effective resistance between the end points of an edge is proportional to the probability of the edge being included in a random spanning tree of the graph \cite{biggs1997algebraic}.
Furthermore, effective resistance is closely related to the Cheeger constant for graphs that measures bottlenecks in graphs \cite{memoli2022persistent}. Because of its various connections to many other objects (e.g., random walks and Laplacians), effective resistance has been widely used in practice; e.g., \cite{spielman2011sparsifiers,alev2018graph,ahmad2021skeleton}.

These properties suggest that the effective resistance is a measure of how ``well-connected'' two nodes are (see \Cref{sec:effective resistance}). In this paper, we will show that the effective resistance can also be used to bound the amount oversquashing between two nodes in a GNN. In particular, the lower the effective resistance between a pair of nodes, the less oversquashing is experienced by a graph neural network sending messages between these nodes.

{\bf Contributions.}
In this paper, we propose to use effective resistance as a way to quantify oversquashing in graph neural networks. We then show how this perspective can be used to modify input graphs to alleviate oversquashing.
\begin{itemize}
        \item In \Cref{sec:effective resistance}, we prove that the information passed from one node to another by any number of layers of a GNN is upper bounded by a quantity related to the effective resistance between the nodes.
        \item In \Cref{sec:rewiring total}, we utilize total effective resistance as a global measure of oversquashing and develop a rewiring algorithm for minimizing total effective resistance by adding edges to the graphs.
        \item In \Cref{sec:exp}, we empirically demonstrate that our rewiring technique is effective in alleviating oversquashing. Our method outperforms the curvature based method SDRF from \cite{topping2022oversquashing} and has similar performance compared to the spectral gap based method FoSR from \cite{karhadkar2022firstorder}.
\end{itemize}
All missing technical details and proofs are in the Appendix.

\paragraph{More on related work.}
\citet{alon2021bottleneck} were the first to study the oversquashing problem in GNNs, although they did not provide a theoretical analysis of the problem. \citet{topping2022oversquashing} were the first to introduce a method for quantitatively analyzing the oversquashing problem. Inspired by \citet{xu2018jumpingknowledge}, Topping et al.~proposed using norm of the Jacobian between node features at different levels of a GNN as a measure of oversquashing. Intuitively the norm of the Jacobian represents the ability of the features at one node to influence the features at another. They proved an upper bound on the norm of the Jacobian for certain nodes by the Balanced Forman Curvature of an edge. However, their theoretical analysis has the limitation that their final upper bound of the Jacobian via curvature only applies to nodes within 2-hop neighborhoods. In contrast, our analysis (\Cref{lem:bound_on_Jacobian} and \Cref{thm:effective_resistance_bound_on_jacobian}) applies to any two nodes at any layer of the GNN. \citet{banerjee2022information} proposed an approach for analyzing the oversquashing problem using techniques from information theory.

\citet{digiovanni2023oversquashing} also analyzed oversquashing using the commute time between a pair of nodes in a concurrent work. Both ours and their papers use similar approaches and reach the conclusion that large effective resistance between a pair of nodes results in more oversquashing. Additionally, they provide an analysis of how the width and depth of a GNN affect oversquashing.

In addition to analyzing the oversquashing problem, there has also been a line of research on ways to alleviate oversquashing. One of the most popular approaches is \textit{rewiring} the graph: adding, removing, or reweighting the edges of the graph to improve the topology of the graph. For example, \citet{alon2021bottleneck} proposed using a fully connected graph in the last layer of a GNN.

A popular, generic approach to rewiring is to optimize some quantity measuring the graph topology. For example, \citet{topping2022oversquashing} proposed a rewiring technique to alleviate the oversquashing problem by increasing the curvature of edges in the graph. However, the most common approach has been to try to increase the \textit{spectral gap} of the graph: the smallest eigenvalue of the Laplacian. Intuitively, the spectral gap is proportional to bottlenecks of graphs through the Cheeger inequality \cite{chung1996laplacians}, so increasing the spectral gap decreases the bottleneck. However, there was previously no theoretical work directly tying the spectral gap to oversquashing (see \Cref{sec:spectral gap}). Some approaches to decrease the spectral gap have been to add edges \cite{karhadkar2022firstorder}, flip edges \cite{banerjee2022information}, reweight edges \cite{arnaiz2022diffwire}, or use an expander to perform a GNN layer \cite{deac2022expander}. Our rewiring technique is most similar to the approach of \citet{karhadkar2022firstorder}: we add edges to minimize the total effective resistance. Conceptually speaking, however, our approach may lead to better results as the total effective resistance reflects the entire spectrum of the graph Laplacian, including the spectral gap. See our discussion in \Cref{sec:spectral gap}.

Particularly relevant to this paper are rewiring techniques that incorporate information about effective resistance \cite{arnaiz2022diffwire, banerjee2022information}. These papers observe that edges with high effective resistance often appear in the bottleneck of the graph, so they target these edges in different ways. \citet{banerjee2022information} flip edges with probability proportional to their effective resistance to increase the spectral gap. \citet{arnaiz2022diffwire} reweight edges proportionally to their effective resistance. While our paper and these papers both study effective resistance as it relates to oversquashing, we make different observations about the relationship between oversquashing and effective resistance. In short, these papers observes that edges of high effective resistance are important to the global topology of the graph so propose to target these edges. In contrast, our paper observes that oversquashing is in part the result of pairs of vertices with high effective resistance so propose to decrease total resistance. In particular, while the approach of \citet{arnaiz2022diffwire} is effective, its effectiveness can not be attributed to decreasing total resistance, as the reweighted graph will have approximately the same effective resistance between all pairs of nodes as the original graph (see Theorem 1 of \cite{arnaiz2022diffwire}.)

Additionally, while not a rewiring technique, \citet{velingker2022affinity} propose node and edge features based on effective resistance as a way of incorporating information about the graph topology into GNNs.

\section{Background}

This section reviews some definitions from Spectral Graph Theory; see books by \citet{chung1997spectral} and \citet{spielman2019sagt} for a more thorough introduction.

\subsection{Matrices and Spectra of Graphs.}

Let $G=(V,E)$ be a connected, undirected, unweighted graph with $n$ vertices and $m$ edges. Let $A$ be the \textit{\textbf{adjacency matrix}} and $D$ be the \textit{\textbf{degree matrix}}. The \textit{\textbf{Laplacian}} is $L=D-A$. Additionally, let $\hat{A} = D^{-1/2} A D^{-1/2}$ be the \textit{\textbf{normalized adjacency matrix}} and $\hat{L} = I  - \hat{A} = D^{-1/2} L D^{-1/2}$ be the \textit{\textbf{normalized Laplacian}}.
\par
The matrices $\hat{L}$ and $\hat{A}$ have the same orthonormal basis of eigenvectors $\{z_i : 1\leq i\leq n\}$ (up to choice of basis) but different eigenvalues. The eigenvalues $\lambda_{i}$ of $\hat{L}$ are in the range $[0,2]$, and the eigenvalues of $\hat{A}$ are $\mu_{i} = 1 - \lambda_{i}$, which are in the range $[-1,1]$. The matrix $\hat{A}$ always has eigenvalue $1$ and has eigenvalue $-1$ if and only if $G$ is bipartite. We use the notational convention that $\lambda_n\geq\cdots\geq\lambda_2>\lambda_1=0$ and $\mu_n\leq\cdots\leq\mu_2<\mu_1=1$. $z_1$, the $\mu_1$-eigenvector of $\tilde{A}$ satisfies $z_1(v) = \sqrt{d_v/2m}$, where $d_v$ is the degree of a vertex $v$.

\subsection{Graph Neural Networks}
Consider a graph $G$ with node features $X\in \R^{n\times d}$. We let $x_v\in\R^d$ denote the row in $X$ corresponding to the vertex $v\in V$. A \textit{\textbf{Graph Neural Network}} (GNN) updates the node features by iteratively aggregating features of nodes in the neighborhood.
More precisely, the feature vectors at each layer are iteratively computed by
$$
h^{(0)}_v:=x_v,\,\,h^{(l+1)}_{v} = \phi_{l}\left(h_v^{(l)},\, \sum_{u\in\mathcal{N}(v)}\hat{A}_{uv}\psi_{l}\left(h_u^{(l)}\right) \right)
$$
for learnable functions $\phi_l$ and $\psi_{l}$. Note that this is a strict subset of the more general class of Message-Passing Neural Networks~\cite{gilmer2017neural}.

\paragraph{Relational GNNs.} In the process of graph rewiring, the structure of the underlying graph will be changed.
In order to retain information of the original graph and also exploit the new graph structure induced from graph rewiring, we use \textit{\textbf{relational GNNs}} (R-GNNs) \cite{battaglia2018relational} to accommodate both information. The idea of using R-GNNs for rewired graphs was introduced in \cite{karhadkar2022firstorder}.
In the framework of R-GNNs, for a graph $G$, there exists a set $\mathcal{R}$ of relation types such that each edge $\{u,v\}\in E$ is associated with an edge type $r\in\mathcal{R}$. For each $v\in V$ and $r\in\mathcal{R}$, we let $\mathcal{N}_r(v)\subseteq\mathcal{N}(v)$ denote the collection of all neighbors of $v$ incident to an edge of type $r$.
An R-GNN is a function of the form
$$h^{(l+1)}_{v} = \phi_{l}\left(h_v^{(l)},\sum_{r\in\mathcal{R}}\sum_{j\in\mathcal{N}_r(v)}\hat{A}_{uv}\psi_{l}^r\left(h_u^{(l)}\right) \right)
$$
for learnable functions $\phi_l$ and $\psi_{l}^r$.

\section{Effective Resistance and Oversquashing}\label{sec:effective resistance}

\begin{figure}[htb!]
    \centering
    \includegraphics[width=\linewidth]{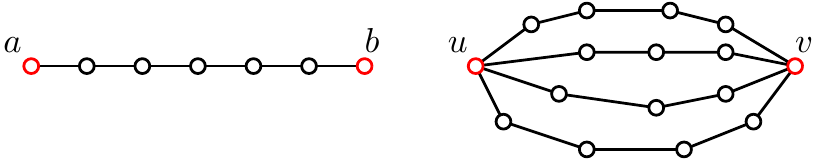}
    \caption{Two examples where effective resistance can be easily computed. For vertices $u$ and $v$ connected by several vertex-disjoint paths $p$, $R_{u,v} = (\sum_{uv\text{-paths }p}\mathrm{length}(p)^{-1})^{-1}$. Left: $R_{a,b}=6$, the length of the path. Right: $R_{u,v} = 10/9$.}
    \label{fig:resistance_examples}
\end{figure}

Let $u$ and $v$ be vertices of $G$. The \textit{\textbf{effective resistance}} between $u$ and $v$ is defined
$$
R_{u,v} = (1_u-1_v)^{T}L^{+}(1_u-1_v),
$$
where $1_v$ is the indicator vector of the vertex $v$ and $L^{+}$ is the pseudoinverse of $L$.
The effective resistance can also be computed using the normalized Laplacian $\hat{L}$. This follows from a formula for effective resistance given by \citet[Corollary 3.2]{lovasz1993random}, but is somewhat non-standard. We provide a different proof in \Cref{apx:alternative_formula_effective_resistance} for completeness.

\begin{restatable}{lemma}{normalizedr}
\label{lm:normalized R}
Let $G$ be a connected graph. Let $u$ and $v$ be two vertices. Then
$$
R_{u,v} = \left(\frac{1}{\sqrt{d_u}}1_u - \frac{1}{\sqrt{d_v}}1_v\right)^{T} \hat{L}^{+}\left(\frac{1}{\sqrt{d_u}}1_u - \frac{1}{\sqrt{d_v}}1_v\right).
$$
\end{restatable}

 Intuitively, the effective resistance is a measure of how ``well-connected'' two vertices $u$ and $v$ are. While ``well-connected''-ness is informal, there are many theorems which suggest such a connection. For example, if $u$ and $v$ are connected by $k$ edge-disjoint paths of length at most $l$, then the effective resistance $R_{u,v}$ is a most $l/k$. Therefore, the more and shorter paths connecting $u$ and $v$, the smaller the effective resistance between $u$ and $v$. See the Introduction for more intuition behind effective resistance.

\subsection{Effective Resistance and the Jacobian of GNNs.}\label{sec: effective jacobian}

As a way of measuring oversquashing in graph neural networks, \citet{topping2022oversquashing} proposed upper bounding the 2-norm of the Jacobian between node features $\|\partial h_{u}^{(r)} / \partial x_v\|$; here, both $h_{u}^{(r)}$ and $x_v$ are vectors, so ${\partial h_{u}^{(r)} / \partial x_v}$ is the Jacobian matrix. The Jacobian captures the influence of initial feature vector $x_v$ at vertex $v$ upon the feature vector $h_{u}^{(r)}$ at vertex $u$ at the $r$\textsuperscript{th} layer of the GNN. A smaller upper bound on the partial derivative indicates that that the features at the node $v$ can have less influence on the features at the node $u$. We adopt this way of analysis and establish a bound on the norm of the Jacobian matrix via the effective resistance.
\par
First, we show how the norm of the Jacobian is upper bounded by the powers of the normalized adjacency matrix.

\begin{restatable}{lemma}{boundonjacobian}
\label{lem:bound_on_Jacobian}
    Let $u,v\in V$ and let $r\in\mathbb{N}$. Assume that $\norm{\nabla\phi_{l}}\leq\alpha$ and $\max\{\norm{\nabla\psi_{l}},1\}\leq\beta$ for all $l=0,\ldots,r$, where $\nabla f$ denotes the Jacobian of a map $f$. Then
    $$
        \norm{\frac{\partial h_{u}^{(r)}}{\partial x_v}}  \leq  (2\alpha\beta)^{r}\sum_{l=0}^{r} (\hat{A}^{l})_{uv}.
    $$
\end{restatable}

This result is different from Lemma 1 in \cite{topping2022oversquashing} in which the two vertices $u$ and $v$ are required to be {\bf exactly} distance $r$ apart from each other; while our result is for any two vertices.

We can now use \Cref{lem:bound_on_Jacobian} to establish a new bound via effective resistance. Recall that $\mu_n\leq\cdots\leq \mu_2<\mu_1=1$ denote the eigenvalues of $\hat{A}$.

\begin{restatable}{theorem}{effectiveresistanceboundonjacobian}
\label{thm:effective_resistance_bound_on_jacobian}
Let $G$ be a non-bipartite graph. Let $u,v\in V$. Let $\norm{\nabla\phi_{l}}\leq\alpha$ and $\max\{\norm{\nabla\psi_{l}},1\}\leq\beta$. Let $\dmin=\min\{d_u,d_v\}$ and $\dmax=\max\{d_u,d_v\}$. Let $\max\{|\mu_2|,|\mu_n|\}\leq \mu$. Then
$$
    \norm{\frac{\partial h^{(r)}_{u}}{\partial x_v}} \leq (2\alpha\beta)^{r}\frac{\dmax}{2}\left(\frac{2}{\dmin}\left(r+1 + \frac{\mu^{r+1}}{1-\mu}\right) {-} R_{u,v}\right)
$$
\end{restatable}

\Cref{thm:effective_resistance_bound_on_jacobian} intuitively suggests that vertices with low effective resistance have a better influence over each other in message passing; that is, the node feature $h_u^{(r)}$ at node $u$ in level $r$ is more affected by the initial node feature $x_v$ at node $v$.
Intuitively this makes sense, as effective resistance is tied to the number and length of paths connecting $u$ and $v$. The more and shorter paths connecting $u$ and $v$, the lower the effective resistance between $u$ and $v$ is. This implies that there are more ways for a GNN to send messages between $u$ and $v$, and indeed, by \Cref{thm:effective_resistance_bound_on_jacobian}, the less oversquashing between $u$ and $v$.

\begin{proof}[Sketch of proof of \Cref{thm:effective_resistance_bound_on_jacobian}]
    \Cref{lem:bound_on_Jacobian} allows us to bound the Jacobian by a sum of entries of powers of the adjacency matrix. Therefore, we need a way of connecting powers of the adjacency matrix to effective resistance. For this, we use the following two lemmas, which themselves may be of independent interest. Detailed proofs of the theorem and the lemmas can be found in \Cref{sec:proof of effective bound}.

    Let $\hat{A}_r$ denote the restriction of $\hat{A}$ to the space orthogonal to the eigenvector $z_1$, i.e.~$\hat{A}_r=\sum_{i=2}^{n} \mu_i z_i z_i^{T}$. Recall that the eigenvalues of $\hat{A}_{r}$ are in the range $[-1,1)$, and $(-1,1)$ if $G$ is not bipartite. The pseudoinverse of $\hat{L}$ can be characterized as follows.

    \begin{restatable}{lemma}{laplacianequalssumofadjacency}
    \label{lem:laplacian_equals_sum_of_adjacency}
        Let $G$ be a connected, non-bipartite graph. Then $\hat{L}^{+} = \sum_{j=0}^{\infty} \hat{A}^{j}_r$.
    \end{restatable}

    This characterization of $\hat{L}^{+}$ allows us to prove the following relationship between the effective resistance and powers of the normalized adjacency matrix $\hat{A}$ (not just $\hat{A}_r$.)

    \begin{restatable}{lemma}{effectiveresistancesumofadjacency}
    \label{lem:effective_resistance_sum_of_adjancency}
        Let $G$ be a non-bipartite graph. Let $u$ and $v$ be two vertices in $G$. Then
        $$R_{u,v} =  \sum_{i=0}^{\infty} \left(
        \frac{1}{d_u}(\hat{A}^{i})_{uu} + \frac{1}{d_v} (\hat{A}^{i})_{vv} - \frac{2}{\sqrt{d_vd_u}}  (\hat{A}^{i})_{uv}\right). $$
    \end{restatable}
    The upper bound in \Cref{thm:effective_resistance_bound_on_jacobian} follows from \Cref{lem:bound_on_Jacobian} and \Cref{lem:effective_resistance_sum_of_adjancency}.
\end{proof}

\paragraph{Total Resistance}

We now take our analysis one step further and summarize message passing rate between {\bf all pairs} of nodes at any given layer of GNN using the notion of \textit{\textbf{total effective resistance}} $\Rtot$---the sum of the effective resistance between all pairs of vertices.
\par
As the partial derivative between a pair of vertices is bounded above by a function of the effective resistance, the total resistance bounds the sum of the Jacobian between all pairs of vertices in the graph. The following corollary follows immediately from \Cref{thm:effective_resistance_bound_on_jacobian}.

\begin{corollary}
\label{cor:total_resistance_bound_on_jacobian}
    Let $G$ be a non-bipartite graph. Let $\norm{\nabla\phi_{l}}\leq\alpha$ and $\max\{\norm{\nabla\psi_{l}},1\}\leq\beta$. Let $\dmin = \min_{v\in V}d_v $ and $\dmax = \max_{v\in V} d_v$. Let $\max\{|\mu_2|,|\mu_n|\}\leq \mu$. Then
    \begin{align*}
        &\sum_{u\neq v\in V} \norm{\frac{\partial h^{(r)}_{u}}{\partial x_{v}}}  \\
       \leq& (2\alpha\beta)^{r}\frac{d_{\max}}{2}\left(\frac{n\cdot(n-1)}{d_{\min}}\left(r+1 + \frac{\mu^{r+1}}{1-\mu}\right) - \Rtot \right).
    \end{align*}
\end{corollary}

\paragraph{Comparison with Curvature Bounds.}

\Cref{thm:effective_resistance_bound_on_jacobian} and \Cref{cor:total_resistance_bound_on_jacobian} are inspired by Theorem 4 in \cite{topping2022oversquashing}, which bounds the Jacobian matrix between vertex features by the Balanced Forman curvature of an edge. In some ways, the effective resistance and Balanced Forman curvature of an edge are similar, as both measure how connected the endpoints are. However, our analysis generalizes the previous bound in several important ways.

(1) Our analysis can be applied to any pair of vertices in a graph, not just those vertices at distance 2.

(2) Effective resistance can be used to bound the oversquashing between node features after an arbitrary number of layers of a GNN, unlike Balanced Forman Curvature which can only measure oversqushing after 2 consecutive layers.

In short, the reason for both of these generalizations is that effective resistance measures the \textit{global} connectivity between a pair of vertices, while Balanced Forman curvature only measures the \textit{local} connectivity between a pair of nodes. See \Cref{fig:curvature vs resistance} for an illustration.

\begin{figure}[ht!]
    \centering
    \includegraphics[width=\linewidth]{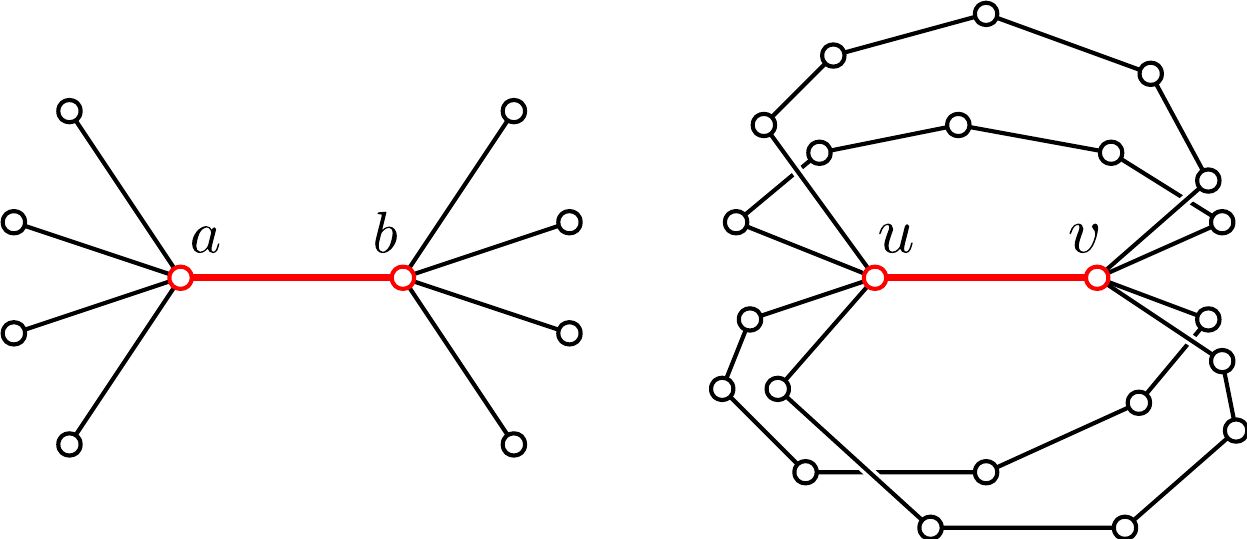}
    \vspace*{-0.1in}\caption{The edges $\{a,b\}$ and $\{u,v\}$ have the same Balanced Forman curvature of $\Ric(a,b) = \Ric(u,v) -6/5$. However, their effective resistance are different ($R_{a,b}=1$ and $R_{u,v} = 3/5$). This shows how the curvature only measure local connectivity and does not distinguish global connectivity as effective resistance does.}
    \label{fig:curvature vs resistance}
\end{figure}

\paragraph{Comparison with Commute Time Bounds}

Concurrently to this work, \citet{digiovanni2023oversquashing} showed that oversquashing between a pair of nodes $u$ and $v$ could be bounded by the commute time $\tau(u,v)$---the expected number of steps in a random walk from $u$ to $v$ and back to $u$. The commute time and effective resistance are proportional: $\tau(u,v)=2m R_{u,v}$~\cite{chandra1996resistance}; thus, our \Cref{thm:effective_resistance_bound_on_jacobian} and their Theorem 5.5 are analogous. Indeed, both theorems agree that oversquashing occurs between nodes with large effective resistance/commute time. The two theorems also use similar techniques to connect effective resistance/commute time to the Jacobian of a GNN. The main differences between our theorems are the result of differences in the quantities we bound (both are related to the Jacobian of the GNN) and differences in assumptions about the GNN.

\subsection{Effective Resistance and the Spectral Gap}\label{sec:spectral gap}

Let $0=\sigma_1\leq\sigma_2\leq\cdots\sigma_n$ denote the eigenvalues of the (un-normalized) Laplacian $L$. The second eigenvalue $\sigma_2$ is called the \textit{\textbf{spectral gap}}\footnotemark  of the graph $G$. The spectral gap is often used as a measure of the ``bottleneck'' of a graph. This is because the spectral gap is proportional to the size of the sparsest cut in the graph, a classic result known as \textit{\textbf{Cheeger's Inequality}} \cite{chung1996laplacians}.

\footnotetext{\label{footnote:normalized_vs_unnormalized} In this section, we focus on the spectral gap and eigenvalues of the \textit{unnormalized} Laplacian, while previous papers studying oversquashing have focused on the spectral gap of the \textit{normalized} Laplacian. There are variants of Cheeger's inequality for both the normalized and unnormalized spectral gap~\cite{chung1997spectral}, so both spectral gaps provide a measure of the connectivity and bottleneck of a graph. The eigenvalues of $L$ and $\hat{L}$ are also closely related as follows: $d_{\min}\lambda_k\leq\sigma_k\leq d_{\max}\lambda_k$.}

Previous research has attempted to connect oversquashing to the spectral gap of the graph \cite{topping2022oversquashing, banerjee2022information}. This has motivated rewiring heuristics aimed at raising at the spectral gap~\cite{arnaiz2022diffwire, banerjee2022information, deac2022expander, karhadkar2022firstorder}. However, unlike our theoretical analysis for effective resistance (\Cref{thm:effective_resistance_bound_on_jacobian} and \Cref{cor:total_resistance_bound_on_jacobian}), while the use of spectral gap for measuring oversquashing is intuitive, there was previously no theoretical evidence for how the spectral gap directly bounds information passing between nodes.
\par
In this section, we first discuss the connections between spectral gap and effective resistance in order to derive a first-step theoretical justification for using spectral gap for bounding oversquashing. Then, we discuss potential limitations of only using the spectral gap.
\par
The following existing result shows that the worst-case effective resistance between any pair of nodes is proportional to the spectral gap.

\begin{theorem}[Theorem 4.2, \cite{chandra1996resistance}]
\label{thm:spectral_gap_and_total_resistance}
    Let $R_{\max}$ denote the maximum effective resistance between any pair of vertices in $G$. Then
    $$
        \frac{1}{n\sigma_2} \leq R_{\max} \leq \frac{2}{\sigma_2}.
    $$
\end{theorem}

\Cref{cor:total_resistance_bound_on_jacobian} and \Cref{thm:spectral_gap_and_total_resistance} combine to reinforce the idea that low spectral gap is tied to oversquashing, as seen by the following corollary.

\begin{corollary}
\label{cor:spectral_gap_bound_on_jacobian}
Under the same assumptions as in \Cref{cor:total_resistance_bound_on_jacobian}, one has that
    \begin{align*}
        &\sum_{u\neq v\in V} \norm{\frac{\partial h^{(r)}_{u}}{\partial x_{v}} } \\
       \leq & (2\alpha\beta)^{r}\frac{d_{\max}}{2}\left(\frac{n\cdot(n-1)}{d_{\min}}\left(r+1 + \frac{\mu^{r+1}}{1-\mu}\right) - \frac{1}{n\sigma_2} \right).
    \end{align*}
\end{corollary}

Of course, the bound above is looser than the bound using $\Rtot$ in \Cref{cor:total_resistance_bound_on_jacobian}. Furthermore, the following result suggests that oversquashing behavior of the graph is tied not just to the spectral gap, but rather to the \textit{entire} spectrum of the Laplacian. Therefore, raising the entire spectrum of the Laplacian, not just the spectral gap, could potentially further reduce oversquashing.

\begin{theorem}[Section 2.5, \cite{ghosh2008minimizing}]\label{lm: Rtot}
    Let $G$ be a connected graph with $n$ vertices, Laplacian $L$, and total resistance $\Rtot$. Then
    $$
        \Rtot = n\cdot\tr L^{+} = n \sum_{i=2}^{n}\frac{1}{\sigma_i}
    $$
\end{theorem}

The higher eigenvalues of $L$ also carry topological meaning about the graph. Just as the spectral gap $\lambda_2$ measures the obstruction to bipartitioning a graph (the ``bottleneck''), the $k$\textsuperscript{th} smallest eigenvalue $\lambda_k$ of $\hat{L}$ is related to partitioning a graph into $k$ parts \cite{lee2014multiway}. See \Cref{footnote:normalized_vs_unnormalized} for the relationship between the eigenvalues $\lambda_k$ and $\sigma_k$.

\section{Minimizing Total Resistance by Rewiring}\label{sec:rewiring total}

\begin{figure*}[th!]
    \centering
    \begin{subfigure}{0.33\textwidth}
        \centering
        \includegraphics[width=\linewidth]{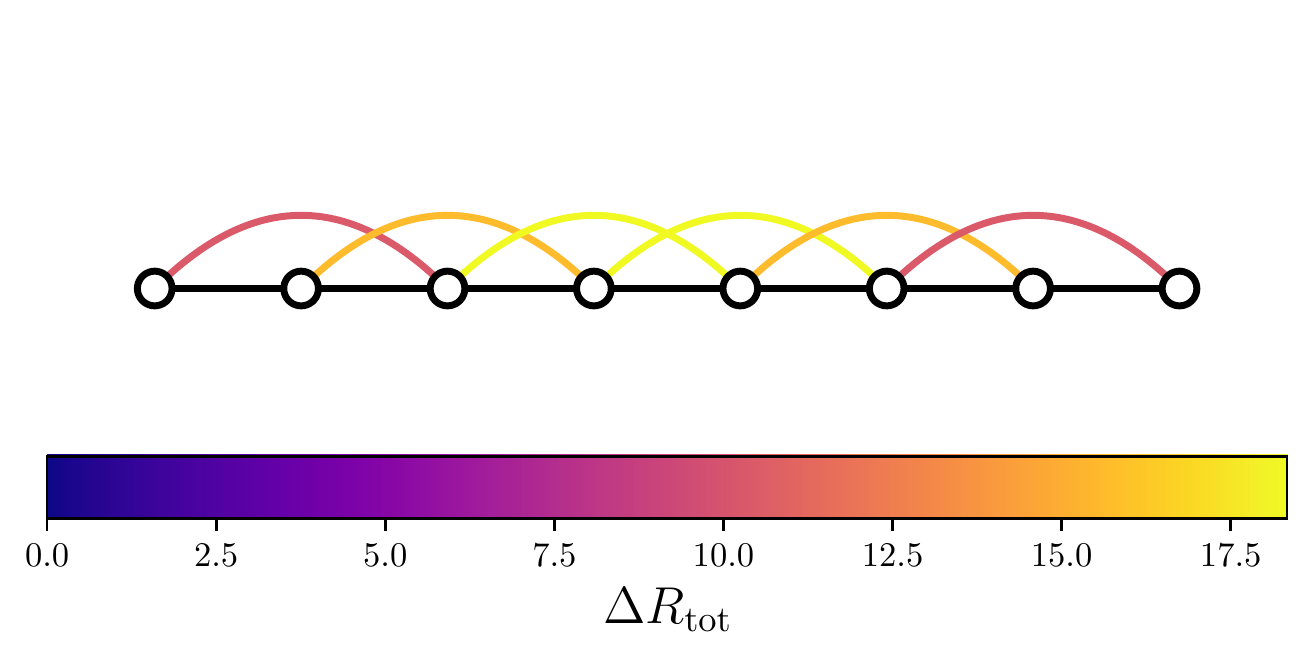}
    \end{subfigure}
    \begin{subfigure}{0.33\textwidth}
        \centering
        \includegraphics[width=\linewidth]{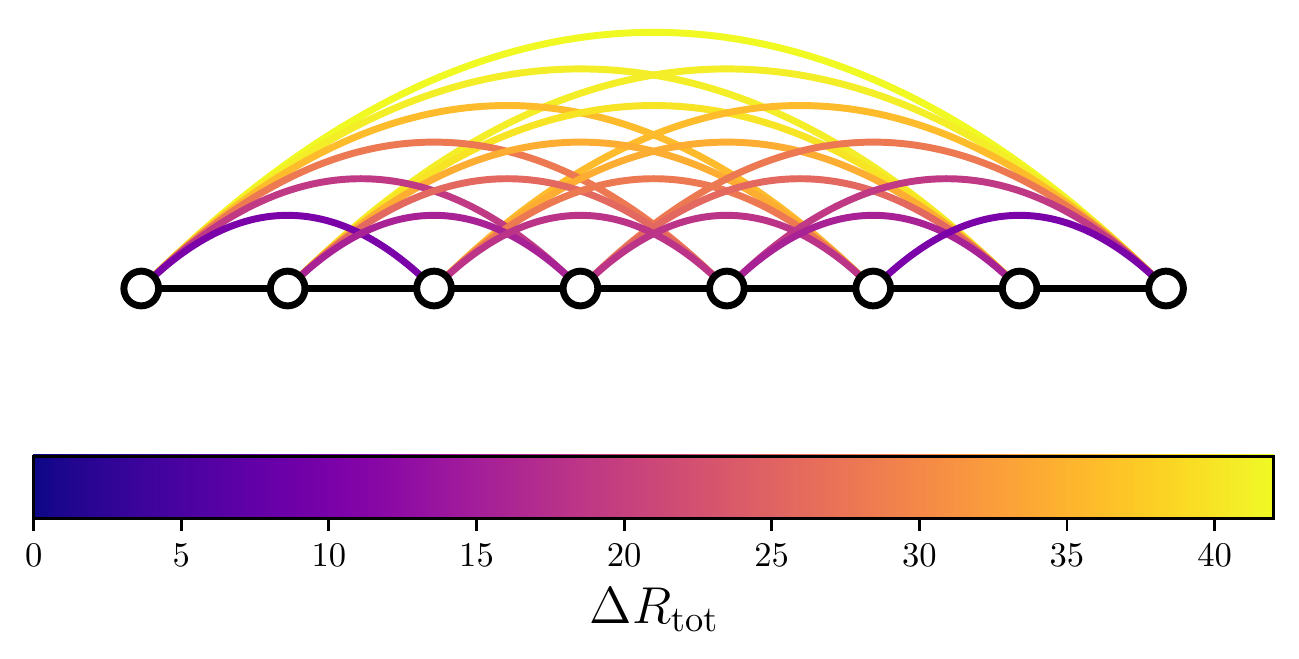}
    \end{subfigure}
    \begin{subfigure}{0.33\textwidth}
        \centering
        \includegraphics[width=\linewidth]{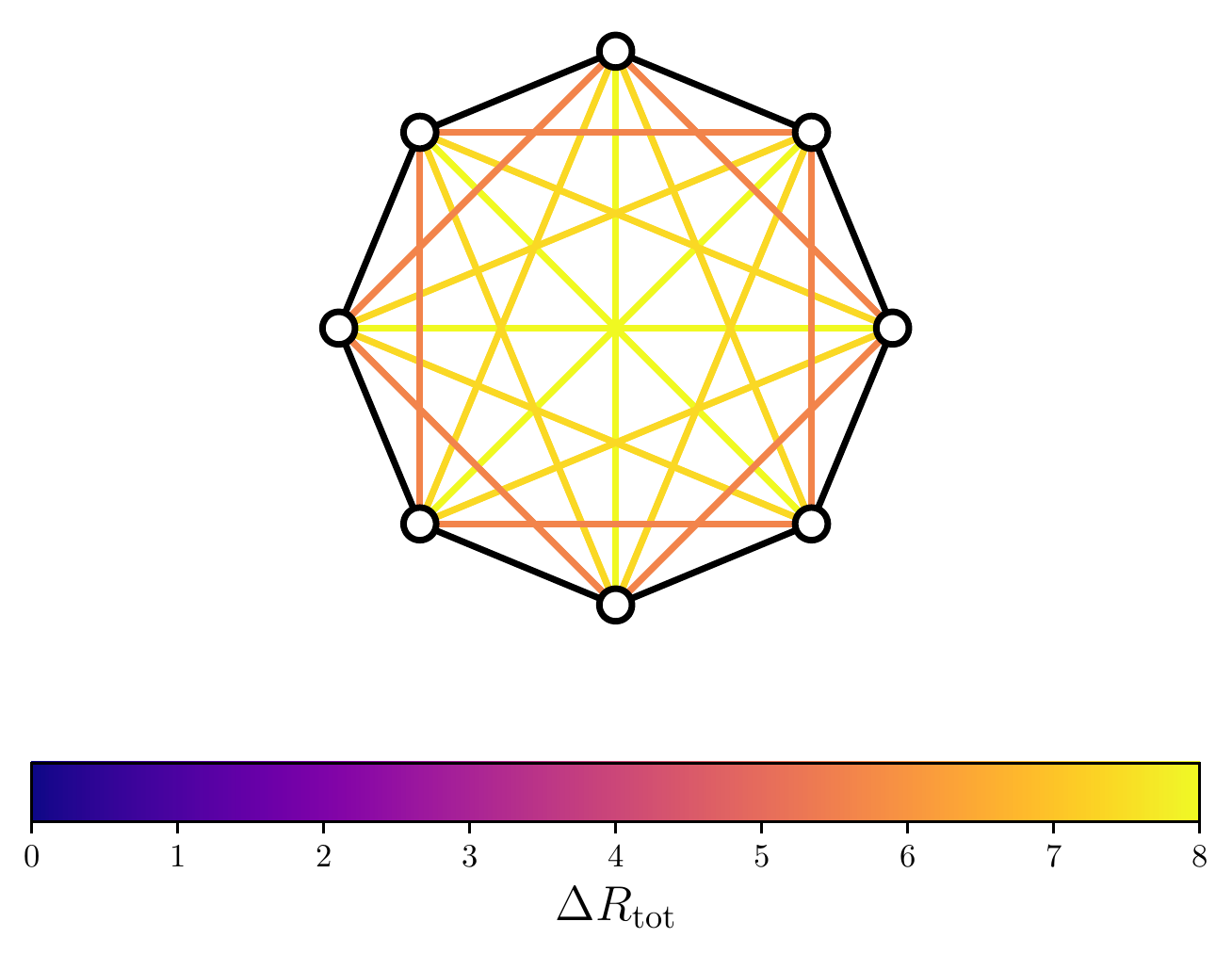}
    \end{subfigure}
    \caption{When an edge $\{u,v\}$ is added to a graph, it decreases the total resistance by $\Delta \Rtot:=n\cdot(B^{2}_{u,v}/(1+R_{u,v}))$ (\Cref{thm:change_in_total_resistance}). This figure shows the value $\Delta \Rtot$ for various pairs of vertices in graphs with $n=8$ vertices. Edges originally in the graph are black, and edges not in the graph are colored according to $\Delta \Rtot$. Left: For pairs of vertices with equal effective resistance $R_{u,v}=2$, edges towards the center of the graph have the highest biharmonic distance $B_{u,v}$. Center: The pairs of vertices that maximize $\Delta \Rtot$ are those at opposite ends of the path. Right: The pairs of vertices that maximize $\Delta \Rtot$ are on opposite sides of the cycle.}
    \label{fig:buv_over_1_plus_ruv}
\end{figure*}

 Motivated by \Cref{cor:total_resistance_bound_on_jacobian}, we propose to address oversquashing by ``rewiring'' a graph to minimize its total resistance. Adding any edge to the graph will decrease its total resistance (a result known as \textbf{\textit{Rayleigh Monotonicity}}), so in this section, we (1) derive a formula to determine how much adding a specific edge decreases the total resistance and (2) propose a rewiring method that greedily adds the edge to the graph that most decreases total resistance. Note that our ``rewiring'' just refers to adding edges, while some previous usage of the term ``rewiring'' might refer to replacing one edge with another \cite{topping2022oversquashing,banerjee2022information}.

\paragraph{Change to $\Rtot$ after adding one edge.} We first need a new notion. The \textit{\textbf{biharmonic distance}} between a pair of vertices $u$ and $v$ is
$$
B_{u,v} = \sqrt{(1_u-1_v)^{T} (L^{+})^{2} (1_u - 1_v)}.
$$
The biharmonic distance was first introduced in the context of geometry processing \cite{lipman2010biharmonic}. However, before it was properly named, it was discovered that the squared biharmonic distance between $u$ and $v$ is proportional to the partial derivative of the total resistance with respect to the weight of the edge $\{u,v\}$, i.e. $\partial R_{tot}/\partial w_{u,v} = -n\cdot B^{2}_{u,v}$~\cite{ghosh2008minimizing}. This suggests that the biharmonic distance can be used as a measure for the effect an edge has on the global connectivity of the graph.
\par
The following theorem may be seen as the unweighted and combinatorial analogue of the previous result (but is proved using completely different means.) This theorem allows us to calculate how much the total resistance decreases when an (unweighted) edge $\{u,v\}$ is added to the graph.

\begin{restatable}{theorem}{changeintotalresistance}
\label{thm:change_in_total_resistance}
   Let $G$ be a connected graph with $n$ vertices. Let $\{u,v\}$ be an edge not in $G$. The difference in total resistance after adding the edge $\{u,v\}$ to $G$ is
    $$
         \Rtot(G) - \Rtot(G\cup\{u,v\}) = n\cdot \frac{B^{2}_{u,v}}{1+R_{u,v}}
    $$
\end{restatable}

\begin{proof}[Sketch of proof of \Cref{thm:change_in_total_resistance}]

    Note that adding the edge $\{u,v\}$ to $G$ changes the Laplacian from $L$ to $L+(1_u-1_v)(1_u-1_v)^{T}$. Hence by \Cref{lm: Rtot} we need to compare the traces of the pseudoinverses of $L$ and $L+(1_u-1_v)(1_u-1_v)^{T}$. This naturally leads us Woodbury's formula:
    \begin{lemma}[Woodbury's Formula]\label{lm:woodbury}
        Let $A$ be an invertible matrix. Let $x$ be a vector. Then
        $$
            (A+xx^{T})^{-1} = A^{-1} - A^{-1}x(1 + x^{T} A^{-1} x)^{-1}x^{T}A^{-1}.
        $$
    \end{lemma}
    As $L$ is singular, we cannot apply Woodbury's Formula directly to $L+(1_u-1_v)(1_u-1_v)^{T}$. Hence, we consider the variant of the Laplacian $L+\frac{11^{T}}{n}$, where $1$ is the all-ones vector. If $G$ is connected, then $L+\frac{11^{T}}{n}$ is invertible. Moreover, it can be shown that
    \begin{restatable}[\cite{ghosh2008minimizing}]{lemma}{modifiedformula}
    \label{lm:modified formula}
    Let $G$ be a connected graph. Then
        \begin{itemize}
            \item $R_{u,v} = (1_u - 1_v)^T (L+\frac{11^{T}}{n})^{-1} (1_u - 1_v)$;
            \item $B_{u,v}^{2} = (1_u - 1_v)^T (L+\frac{11^{T}}{n})^{-2} (1_u - 1_v)$;
            \item $\Rtot = n\cdot\tr(L+\frac{11^{T}}{n})^{-1}-n$.
        \end{itemize}
    \end{restatable}
    We can therefore apply \Cref{lm:woodbury} to compute $(L+\frac{11^{T}}{n}+(1_u-1_v)(1_u-1_v)^{T})^{-1}$, take the trace, and conclude the theorem. See \Cref{apx:change_in_total_resistance} for all the details.
\end{proof}

\Cref{fig:buv_over_1_plus_ruv} shows the value $n\cdot \frac{B^{2}_{u,v}}{1+R_{u,v}}$ for edges in various graphs.

\paragraph{Rewiring heuristic.}

Motivated by \Cref{thm:change_in_total_resistance}, we propose the following heuristic, \textbf{\textit{Greedy Total Resistance (GTR) rewiring}}, to minimize the total resistance: repeatedly add the edge $\{u,v\}$ that maximizes $B^{2}_{u,v}/(1+R_{u,v})$. For disconnected graphs, the effective resistance and biharmonic distance between vertices in different components is not meaningful. Therefore, we only add edges between vertices that are already in the same connected component.
While we could also use \Cref{thm:change_in_total_resistance} to determine which edge to remove to most decrease the total resistance, we will only add edges in this paper. A PyTorch Geometric implementation of the GTR algorithm is available online\footnote{\url{https://github.com/blackmit/gtr_rewiring}}. See \Cref{apx:total_resistance_curves} for plots of how much GTR decreases total resistance for various datasets.

\paragraph{Time complexity.}

GTR can naively be implemented in $O(n^{3})$ time, but there are more sophisticated algorithms that take time $O(m\poly\log n + n^2\poly\log n)$. See \Cref{apx:runtime} for an asymptotic and empirical analysis of its runtime.

\paragraph{Adding multiple edges.}

While Theorem \ref{thm:change_in_total_resistance} tells us which single edge most decreases the total resistance when added to the graph, unfortunately, we cannot use this formula to determine which set of $k\geq 2$ edges most decrease the total resistance of the graph. In \Cref{sec:counterexample}, we give an example of a graph where the two edges that most decrease the total resistance are not the two edges that maximize the formula in \Cref{thm:change_in_total_resistance}.

Another challenge for {designing recursive algorithms} to add multiple edges is that the amount an edge decreases the total resistance is \textit{\textbf{non-monotonic}} with respect to subgraphs. By non-monotonic, we mean that for nested graphs $H\subset G$, the amount an edge decreases the total resistance when added to $G$ can be \textit{more} than the amount the same edge would decrease the total resistance when added to $H$. \Cref{sec:counterexample} gives an example where this is the case. Intuitively, this means that an edge can become more important to the global topology of a graph when more edges are added. This is in contrast to the effective resistance, which only decreases with the addition of more edges.
\par
The best algorithm we know for computing the set of $k$ edges that most decrease the total resistance is a brute-force search over all $O(\binom{\binom{n}{2}}{k})$ sets of $k$ edges. It was recently shown that finding the $k$ edges that most decrease the total resistance is NP-Hard~\cite{kooij2023minimizing}. Because of this, it is reasonable to use a heuristic rather than exactly compute the best edges to add to decrease total resistance.

\section{Experiments}\label{sec:exp}

We primarily compare our new GTR rewiring algorithm with the \textbf{\textit{FoSR}} (for ``first-order spectral rewiring'') algorithm proposed by \citet{karhadkar2022firstorder}, as FoSR is the rewiring strategy with the best performance. {{FoSR}} aims at reducing oversquashing in graphs by increasing the spectral gap. FoSR is perhaps the rewiring heuristic most similar to GTR for two reasons. First, it only changes the topology of the graph by adding edges. Second, it is designed to increase the spectral gap of the graph, which will necessarily increase the total resistance of the graph.

 \subsection{Spectral Gap vs. Total Resistance}

\begin{figure}[ht]
    \centering
    \begin{subfigure}[c]{0.5\linewidth}
    \centering
    \includegraphics[height=1.1in]{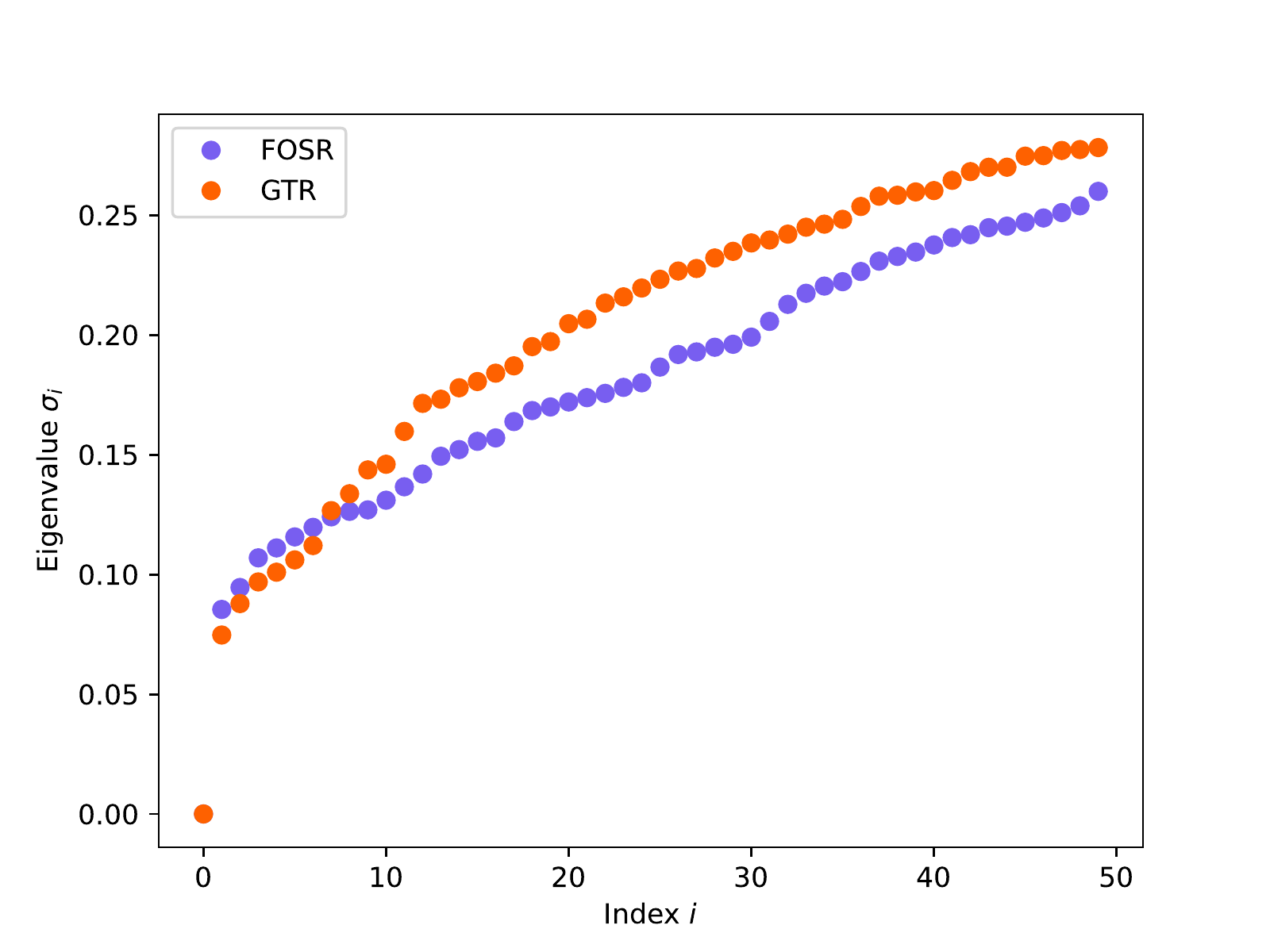}
    \end{subfigure}%
    \begin{subfigure}[c]{0.5\linewidth}
        \centering
        \begin{tabular}{c|c c}
            & FoSR & GTR \\
            \hline
            $\sigma_2$ & 0.085 & 0.075 \\
            $\Rtot$  & 4250377 & 4114024
        \end{tabular}
    \end{subfigure}
   \vspace*{-0.1in}
   \caption{A comparision of the largest connected component of Cora after adding 50 edges with FoSR and GTR. Left: A plot of the smallest 50 eigenvalues of the Laplacian. Right: The spectral gap and total resistance.}
    \vspace*{-0.1in}\label{fig:gtr_vs_fosr_eigenvalues}
\end{figure}

To compare FoSR and GTR, we use both methods to add 50 edges to the largest connected component of the Cora citation network~\cite{mccallum2000cora}. Figure~\ref{fig:gtr_vs_fosr_eigenvalues} shows the 50 smallest eigenvalue after rewiring. FoSR increases the first few eigenvalues (including the spectral gap) more, while GTR increases the larger eigenvalues more. In total, GTR does more to decrease the total resistance of the graph.

\subsection{Graph Classification}
\label{sec:graph_classication}

\begin{table*}[ht!]
\centering
\caption{
    Results of different combinations of rewiring and convolutions on different graph classification datasets. {\color{ibm-orange} \textbf{First}}, {\color{ibm-magenta} \textbf{second}}, and {\color{ibm-purple} \textbf{third}} best results are colored. See \Cref{sec:graph_classication} for discussion. All results except for GTR are from \cite{karhadkar2022firstorder}, with the exception of R-GIN FoSR results marked with an asterisk $({*})$; these are the best runs from the edge ablation study (\Cref{sec:edge_ablation}).
}
\begin{tabular}{c c c c c c c}
    \hline
    &&& GCN &&& \\
    \hline
    Rewiring & Mutag & Proteins & Enzymes & Reddit-Binary & IMDB-Binary & Collab \\
    \hline
    None & $72.15\pm2.44$ & $70.98\pm0.74$ & $27.67\pm 1.16$ & $68.26\pm1.10$ & $49.77\pm0.82$ & $33.78\pm0.49$  \\
    Last FA & $70.05\pm 2.03$ & $71.02\pm 0.96$ & $26.47\pm 1.20$ & $68.49\pm0.95$ & $48.98\pm0.95$ & $33.32\pm 0.44$ \\
    Every FA & $70.45\pm 1.96$ & $60.04\pm 0.93$ & $18.33\pm 1.04$ & $48.49\pm 1.04$ & $48.17\pm 0.80$ & $51.80\pm 0.42$ \\
    DIGL & $79.70\pm 2.15$ & $70.76\pm 0.77$ & $35.72\pm 1.12$ & $76.04\pm 0.78$ & $64.39\pm 0.91$ & $54.50\pm 0.41$ \\
    SDRF & $71.05\pm1.87$ & $70.92\pm0.79$ & $28.37\pm 1.17$ & $68.62\pm0.85$ & $49.40\pm0.90$ & $33.45\pm0.47$  \\
    FoSR & $80.00\pm1.57$ & $73.42\pm0.81$ & $25.07\pm 0.994$ & $70.33\pm0.72$ & $49.66\pm0.86$ & $33.84\pm0.58$  \\
    \hline
    GTR & $79.10\pm 1.86$ & $72.59\pm 2.48$ & $27.52\pm0.99$ & $68.99\pm 0.61$ & $49.92 \pm 0.99$ & $33.05\pm0.40$ \\
    \hline
    &&& R-GCN &&& \\
    \hline
    Rewiring & Mutag & Proteins & Enzymes & Reddit-Binary & IMDB-Binary & Collab \\
    \hline
    None & $69.25\pm2.09$ & $69.52\pm0.73$ & $28.60\pm 1.19$ & $49.85\pm0.65$ & $50.01\pm0.92$ & $33.60\pm1.05$  \\
    Last FA & $70.55\pm 1.81$ & $69.53\pm 0.82$ & $28.23\pm 1.14$ & $49.80\pm 0.63$ & $50.65\pm 0.96$ & $34.73\pm 1.19$ \\
    Every FA & $70.50\pm 1.84$ & $71.67\pm 0.88$ & $33.40\pm 1.14$ & $49.95\pm 0.59$ & $50.50\pm 0.89$ & $33.62\pm 0.98$ \\
    DIGL & $73.40\pm2.00$ & $68.23\pm0.85$ & $28.28\pm1.21$ & $50.00\pm 0.62$ & $49.67 \pm 0.84$  &  $16.93 \pm 1.44$\\
    SDRF & $72.30\pm2.22$ & $69.11\pm0.76$ & $33.48\pm 1.25$ & $58.62\pm0.65$ & $53.64\pm1.04$ & $67.99\pm0.39$  \\
    FoSR & $84.45\pm1.57$ &  $73.80\pm0.69$ & $35.66\pm 1.151$ & $76.59\pm 0.53$ & $64.05\pm 1.12$ & $70.65\pm0.48$ \\
    \hline
    GTR & $\third{85.50\pm 1.47}$ & $\first{75.78\pm0.76}$ & $41.33\pm1.28$ & $80.18\pm0.60$ & $65.09 \pm 0.93$ & $74.34\pm0.41$ \\
    \hline
    &&& GIN &&& \\
    \hline
    Rewiring & Mutag & Proteins & Enzymes & Reddit-Binary & IMDB-Binary & Collab \\
    \hline
    None & $77.70\pm3.60$ & $70.80\pm0.83$ & $33.80\pm 1.12$ & $86.79\pm1.06$ & $70.18\pm0.99$ & $72.99\pm0.38$  \\
    Last FA & $83.45 \pm 1.74$  & $72.30 \pm 0.67$ & $47.40 \pm 1.39$ & $90.22 \pm 0.48$ & $70.91 \pm 0.79$  & $75.06 \pm 0.41$ \\
    Every FA & $72.55\pm3.02$ & $70.38\pm0.91$ & $28.38\pm1.05$ & $50.36\pm0.65$ & $49.16\pm0.87$ & $32.89\pm0.39$\\
    DIGL & $79.70\pm2.15$ & $70.76\pm0.77$ & $35.72\pm1.20$ & $76.04\pm0.77$ & $64.39\pm0.91$ & $54.50\pm0.41$\\
    SDRF & $78.40\pm2.80$ & $69.81\pm0.79$ & $35.82\pm 1.09$ & $86.44\pm0.59$ & $69.72\pm1.15$ & $72.96\pm0.42$  \\
    FoSR & $78.00\pm 2.22$ &  $75.11\pm 0.82$ & $29.20\pm 1.38$ & $87.35\pm 0.60$ & $71.21\pm 0.92$ & $73.28\pm0.42$ \\
    \hline
    GTR & $77.60\pm2.84$ & $73.13\pm0.69$ & $30.57\pm1.42$ & $86.98 \pm 0.66$ & $71.28\pm0.86$ & $72.93\pm0.42$ \\
    \hline
    &&& R-GIN &&& \\
    \hline
    Rewiring & Mutag & Proteins & Enzymes & Reddit-Binary & IMDB-Binary & Collab \\
    \hline
    None & $83.05\pm1.44$ & $70.50\pm0.81$ & $39.02\pm 1.17$ & $87.97\pm0.56$ & $68.89\pm0.87$ & $75.54\pm0.32$  \\
    Last FA & $80.60\pm 1.64$ & $70.30\pm0.84$ & $\third{48.18\pm1.40}$ & $\third{90.00\pm0.65}$ & $69.71\pm1.03$ & $75.43\pm0.49$\\
    Every FA & $83.05\pm1.52$ & $71.05\pm0.91$ & $\first{54.95\pm1.33}$ & $56.86\pm0.94$ & $\third{71.48\pm0.88}$ & $75.43\pm0.48$\\
    DIGL & $81.45\pm1.49$ & $71.31\pm0.76$ & $37.60\pm1.20$ & $74.43\pm0.72$ & $63.93\pm0.95$ & $54.71\pm0.42$\\
    SDRF & $82.70\pm1.78$ & $70.70\pm0.82$ & $39.58\pm 1.33$ & $86.83\pm0.52$ & $70.21\pm0.81$ & $\third{76.48\pm0.39}$  \\
    FoSR & $\first{86.15\pm 1.49}$ & $\third{75.25\pm0.86}^{*}$ & $45.55\pm0.13$ & $\first{90.94\pm 0.47}^{*}$ & $\first{71.96\pm 0.69}^{*}$ & $\second{77.20\pm 0.38}^{*}$ \\
    \hline
    GTR & $\second{86.10\pm 1.76}$ & $\second{75.64\pm0.74}$ & $\second{50.03\pm 1.32}$ & $\second{90.41\pm 0.41}$ & $\second{71.49\pm 0.93}$ & $\first{77.45\pm.039}$ \\
    \hline
\end{tabular}
\label{table:graph_classification_results}
\end{table*}

We evaluate our rewiring heuristic, GTR, as a preprocessing step for training a graph neural network to perform graph classification.
We compare GTR with the following rewiring method: making the last layer fully connected
(Last FA) and making all layers fully connected (All FA) from \cite{alon2021bottleneck},  DIGL from \cite{gasteiger2019digl}, SDRF from \cite{topping2022oversquashing}, and FOSR
 \cite{karhadkar2022firstorder}. We also report results for no rewiring (None). We conduct the same experiment as in \cite{karhadkar2022firstorder} for GTR; see \Cref{table:graph_classification_results} for results. All results except for GTR and those marked with an asterisk are taken from Table 1 of \cite{karhadkar2022firstorder}.

{\bf Datasets.} We test GTR on the same set of graph classification benchmarks as \citet{karhadkar2022firstorder}. All datasets are from the TUDataset~\cite{morris2020tudataset}.

{\bf Experiments.} We compare four types of graph convolutions: GCN~\cite{kipf2017gcn}, Relational-GCN (R-GCN)~\cite{battaglia2018relational}, GIN~\cite{xu2018gin}, and Relational-GIN (R-GIN). Relational graph neural networks perform different aggregation steps for edges of different types. In the case of GTR, we use two edge types: original graph edges and new edges added by the rewiring algorithms. We tune the number of edges added by GTR and fix all other hyperparameters. Full experimental details can be found in \Cref{sec:experimental_details}.

{\bf Results.}
Test accuracies are presented in Table \ref{table:graph_classification_results} and the number of edges added for each graph are reported in \Cref{sec:experimental_details}. We observe the following:
(1) In general, both our GTR and FoSR outperform the rewiring strategies DIGL, SDRF, or no rewiring at all. In particular, for the case of relational versions of GNNs (i.e., R-GCN and R-GIN), these two approaches often out-perform no-rewiring or SDRF by a large margin. Note that SDRF adds edges based on a local curvature criterion; while both FoSR and our GTR can add any edges, taking the global connectivity of graph into account. Table \ref{table:graph_classification_results} shows that both global strategies outperform the local SDRF, especially for the relation-GNN cases.
(2) The performance of our GTR and FoSR are similar for the GIN and R-GIN architectures. On R-GCN however, GTR not only outperforms FoSR, but often by a large margin.

\subsection{Edge Ablation}
\label{sec:edge_ablation}

In \Cref{apx:edge_ablation}, we repeat the experiment from \Cref{sec:graph_classication} but vary the number of edges added. In particular, our experiments suggest that \textit{there is no optimal number of edges to add} that works across datasets. Moreover, \textit{performance does not necessarily increase as total resistance decreases}, which we can see by comparing FoSR and GTR to Every Layer FA in \Cref{table:graph_classification_results}. Therefore, we recommend treating the number of edges added as a hyperparameter to be tuned during training.

\subsection{Hidden Dimension Ablation}
\label{sec:hidden_dim_ablation}

Another method for address oversquashing is to increase the hidden dimension of the GNN~\cite{alon2021bottleneck, digiovanni2023oversquashing}. To compare this method with rewiring, in \Cref{apx:hidden_dim_ablation}, we repeat the experiment from \Cref{sec:graph_classication} but vary both the number of edges added and the hidden dimension. We conclude that \textit{rewiring and increasing the hidden dimension are complementary methods for addressing oversquashing}, as doing either or both increases the performance of GNNs.

\section{Concluding Remarks}

In this paper, we have provided theoretical evidence that effective resistance can be used as a bound on oversquashing between a pair of nodes in a graph, and that the total resistance can be used as a bound of total oversquashing in a graph. We have also empirically demonstrated that lowering total resistance improves the performance of graph neural networks. Indeed, rewiring techniques based on total effective resistance can significantly improve performance of GNN / R-GNNs for graph classification tasks, reinforcing the notion that improving the connectivity of a graph can improve the performance of graph neural networks.

{\bf Limitations and future work.}
We provide theoretical evidence (\Cref{thm:effective_resistance_bound_on_jacobian}) showing that total effective resistance can be used to bound the amount of oversquashing in a graph. This is in contrast to previous work on oversquashing which relates oversquashing to the spectral gap through intuition alone. While we prove that the spectral gap can also be used to bound oversquashing (\Cref{cor:spectral_gap_bound_on_jacobian}), the bound for total resistance is tighter than the bound for the spectral gap.
\par
Despite the theoretical strength of using total resistance over spectral gap for measuring oversquashing, more research is needed to contrast the effects of the two on oversquashing. A challenge to this task is that the total resistance and spectral gap are intimately linked; for example, adding edges to the graph will necessarily both decrease the total resistance and increase the spectral gap. The oversquashing issue becomes more prominent for graphs with long range interactions (e.g., \cite{dwivedi2022long}). Hence it will be interesting to explore a much broader family of graph benchmarks to study the pros and cons of different rewiring methods.

Finally, we also note that currently we employ a greedy approach to identify a collection of edges to be inserted into an input graph as shortcuts. As discussed in \Cref{sec:rewiring total}, finding the $k$ best edges to add to decrease total resistance is NP-Hard~\cite{kooij2023minimizing}, and it is not clear whether such a greedy strategy even leads to an approximation algorithm of selecting the optimal set of $k$ edges to minimizing total effective resistance. We leave the problem of identifying efficient approximation algorithms for the optimal edges or better heuristics for minimizing total effective resistance as a future direction to investigate.

\section*{Acknowledgements}

Mitchell Black and Amir Nayyeri are supported in part by NSF grants CCF-1941086 and CCF-1816442. Zhengchao Wan and Yusu Wang are supported in part by NSF Grants  CCF-2217033 and CCF-2112665, as well as a gift fund from Qualcomm. Mitchell Black would like to thank Yusu Wang for supporting a visit to UCSD where this project was initiated.

\bibliographystyle{icml2023}
\bibliography{example_paper}

\newpage
\appendix
\onecolumn

%%%%%%%%%%%%%%%%%%
\section{Proofs}

\subsection{Proof of \Cref{lm:normalized R}}
\label{apx:alternative_formula_effective_resistance}

\normalizedr*

\begin{proof}
We will prove this using an alternative but well-known characterization of effective resistance in terms of $uv$-flows. First, we must define another matrix associated with a graph. Let $\boundary$ be the $n\times m$ \textit{\textbf{boundary matrix}} of the graph $G$, where $n:=|V|$ and $m:=|E|$. The matrix $\boundary$ is defined such that for an edge $e=\{u,v\}$, the column $\boundary 1_e = 1_u - 1_v$. (The order of $u$ and $v$ is arbitrary for what follows.)
\par
Many of the definitions in this paper can alternatively be expressed in terms of the boundary matrix. The Laplacian can be expressed $L=\boundary\boundary^{T}$, the normalized Laplacian $\hat{L} = D^{-1/2}\boundary(D^{-1/2}\boundary)^{T}$, and the effective resistance $R_{u,v} = \min\{\|f\|^{2} : \boundary f = (1_u-1_v) \}$. Phrased differently, the effective resistance between $u$ and $v$ is the minimum squared-2-norm of any $uv$-flow. This characterization of the effective resistance follows from the general fact that for any matrix $AA^{T}$ and any vector $x\in\im A$ we have that $x^{T}(AA^{T})^{+} x= (A^{+}x)^{T}(A^{+}x)=\min\{\|y\|^{2} : Ay=x\}$. The proof of the current lemma just applies this fact twice.
\begin{align*}
    R_{u,v} =& (1_u-1_v)^{T} L^{+}(1_u-1_v) \\
    =& \min\{\|f\|^{2} : \boundary f = (1_u - 1_v)\} \\
    =& \min\{\|f\|^{2} : D^{-1/2}\boundary f = D^{-1/2}(1_u - 1_v)\} \tag{as $D^{-1/2}$ is bijective}\\
    =& \left(\frac{1}{\sqrt{d_u}}1_u - \frac{1}{\sqrt{d_v}}1_v\right)^{T}\hat{L}^{+}\left(\frac{1}{\sqrt{d_u}}1_u - \frac{1}{\sqrt{d_v}}1_v\right).
\end{align*}
\end{proof}

\subsection{Proof of \Cref{lem:bound_on_Jacobian}}

\boundonjacobian*

\begin{proof}
We prove this by induction on the layer $r$. For the base case of $r=0$, either $u=v$ or $u\neq v$; in the first case
$$\frac{\partial h_{u}^{(0)}}{\partial x_v} = \frac{\partial x_v}{\partial x_v} = \mathrm{Id}_{d\times d},$$
and in the second case,
$$\frac{\partial h_{u}^{(0)}}{\partial x_v} = \frac{\partial x_u}{\partial x_v} = 0_{d\times d}.$$
Therefore,
\begin{equation}\label{eq: Id = 1}
    \norm{\frac{\partial h_{u}^{(0)}}{\partial x_v}}\leq \max\{\norm{\mathrm{Id}_{d\times d}},\norm{0_{d\times d}}\}=1.
\end{equation}
    \par
Assume that the statement holds for some $r\geq 0$. We now prove the inductive case of $r+1$.
    \begin{align*}
        & \norm{\frac{\partial h_{u}^{(r+1)}}{\partial x_v}} = \norm{\nabla_{1}\phi_{r}\cdot\frac{\partial h_{u}^{(r)}}{\partial x_v}+ \nabla_{2}\phi_{r}\cdot\sum_{w\in\mathcal{N}(u)} \hat{A}_{uw}\cdot \nabla\psi_{r}\cdot\frac{\partial h_{w}^{(r)}}{\partial x_v}}
        \\
        \leq & ~\norm{\nabla_{1}\phi_{r}}\cdot\norm{\frac{\partial h_{u}^{(r)}}{\partial x_v}} + \norm{\nabla_{2}\phi_{r}}\cdot\sum_{w\in\mathcal{N}(u)} \hat{A}_{uw}\norm{\nabla\psi_{r}}\cdot\norm{\frac{\partial h_{w}^{(r)}}{\partial x_v} }  \tag{as $\hat{A}_{uw}$ positive $\forall$ $u,\,w$} \\
        \leq&~ \alpha\cdot\norm{\frac{\partial h_{u}^{(r)}}{\partial x_v}} + \alpha\beta\cdot\sum_{w\in\mathcal{N}(u)} \hat{A}_{uw}\cdot\norm{\frac{\partial h_{w}^{(r)}}{\partial x_v}}  \\
        \leq& ~2^{r}(\alpha\beta)^{r+1}\sum_{l=0}^{r} (\hat{A}^{l})_{uv} + 2^{r}(\alpha\beta)^{r+1} \sum_{l=0}^{r}\sum_{w\in\mathcal{N}(u)} \hat{A}_{uw} (\hat{A}^l)_{wv} \tag{induction}\\
        =& ~2^{r}(\alpha\beta)^{r+1}\sum_{l=0}^{r} (\hat{A}^{l})_{uv} + 2^{r}(\alpha\beta)^{r+1} \sum_{l=1}^{r+1} (\hat{A}^{l})_{uv} \tag{definition of matrix multiplication}\\
        \leq& ~(2\alpha\beta)^{r+1} \sum_{l=0}^{r+1} (\hat{A}^{l})_{uv}.
    \end{align*}
Here $\nabla \phi_r=[\nabla_1\phi_r|\nabla_2\phi_r]$ and $\nabla \psi_r$ denote the Jacobian matrices for $\phi_r$ and $\psi_r$, respectively. $\nabla_1\phi_r$ corresponds to partial derivatives w.r.t. the first several arguments in $\phi_{r}$ corresponding to $h_v^{(r)}$ in the formula $\phi_{r}\!\left(h_v^{(r)},\, \sum_{u\in\mathcal{N}(v)}\hat{A}_{uv}\psi_{l}\left(h_u^{(r)}\right) \right)$ and $\nabla_2\phi_r$ is defined similarly. In the second inequality, we used the fact for 2-norm that $\norm{[A|B]}\geq \max\{\norm{A},\norm{B}\}$. In the third inequality, we used the fact that $\beta\geq 1$ and in this way we have that $\alpha\leq \alpha\beta$.
\end{proof}

\subsection{Proof of \Cref{thm:effective_resistance_bound_on_jacobian}}\label{sec:proof of effective bound}
In this section, we provide proofs of \Cref{lem:laplacian_equals_sum_of_adjacency}, \Cref{lem:effective_resistance_sum_of_adjancency} and \Cref{thm:effective_resistance_bound_on_jacobian}.

\laplacianequalssumofadjacency*
\begin{proof}
First, recall that the eigenvalues of $\hat{A}_{r}$ are in the range $(-1,1)$ if $G$ is not bipartite.
Also note that any number $\mu\in (-1,1)$ satisfies $\sum_{j=0}^{\infty} \mu^{j} = \frac{1}{1-\mu}$. We prove the lemma by applying this fact to the spectral decomposition of $\hat{L}^{+}$.
    \begin{align*}
        \hat{L}^{+} =& \sum_{i=2}^{n} \frac{1}{\lambda_i} z_iz_i^{T}
        = \sum_{i=2}^{n} \frac{1}{1-\mu_i} z_iz_i^{T} \\
        =& \sum_{i=2}^{n} (\sum_{j=0}^{\infty}\mu_{i}^{j}) z_iz_i^{T}
        = \sum_{j=0}^{\infty} \hat{A}^{j}_{r}.
    \end{align*}
\end{proof}

Based on \Cref{lem:laplacian_equals_sum_of_adjacency}, we then prove \Cref{lem:effective_resistance_sum_of_adjancency} below.

\effectiveresistancesumofadjacency*
\begin{proof}
    Observe that
   \begin{align*}
   & ~~\big(\frac{1}{\sqrt{d_u}} 1_u - \frac{1}{\sqrt{d_v}} 1_v \big)^{T} \hat{A}_{r}^{i}\big(\frac{1}{\sqrt{d_u}} 1_u - \frac{1}{\sqrt{d_v}} 1_v \big) \\
   =& ~~\big(\frac{1}{\sqrt{d_u}} 1_u - \frac{1}{\sqrt{d_v}} 1_v\big)^{T} \hat{A}^{i} \big(\frac{1}{\sqrt{d_u}} 1_u - \frac{1}{\sqrt{d_v}} 1_v \big)
   \end{align*}
   for all $i\geq 0$ as $(\frac{1}{\sqrt{d_u}} 1_u - \frac{1}{\sqrt{d_v}} 1_v)^{T} z_1 =0$. We use this equation to alternatively express the effective resistance.
    \begin{align*}
        R_{u,v} =& (\frac{1}{\sqrt{d_u}} 1_u - \frac{1}{\sqrt{d_v}} 1_v)^{T} \hat{L}^{+} (\frac{1}{\sqrt{d_u}} 1_u - \frac{1}{\sqrt{d_v}} 1_v) \\
        =& \sum_{i=0}^{\infty} (\frac{1}{\sqrt{d_u}} 1_u - \frac{1}{\sqrt{d_v}} 1_v)^{T} \hat{A}_{r}^{i} (\frac{1}{\sqrt{d_u}} 1_u - \frac{1}{\sqrt{d_v}} 1_v) \tag{\Cref{lem:laplacian_equals_sum_of_adjacency}}\\
        =& \sum_{i=0}^{\infty} (\frac{1}{\sqrt{d_u}} 1_u - \frac{1}{\sqrt{d_v}} 1_v)^{T} \hat{A}^{i} (\frac{1}{\sqrt{d_u}} 1_u - \frac{1}{\sqrt{d_v}} 1_v) \tag{Above observation}\\
        =& \sum_{i=0}^{\infty} \left(\frac{1}{d_u}(\hat{A}^{i})_{uu} + \frac{1}{d_v} (\hat{A}^{i})_{vv} - \frac{2}{\sqrt{d_ud_v}}  (\hat{A}^{i})_{uv}\right)
    \end{align*}
\end{proof}

Now, we finish proving \Cref{thm:effective_resistance_bound_on_jacobian} as follows.

\effectiveresistanceboundonjacobian*

\begin{proof}
Now, we will combine the equation for effective resistance of \Cref{lem:effective_resistance_sum_of_adjancency} with the bound on the Jacobian matrix of \Cref{lem:bound_on_Jacobian}. This gives us the bound
    \begin{align*}
        \norm{\frac{\partial h^{(r)}_{u}}{\partial x_v} }
        &\leq (2\alpha\beta)^{r}\sum_{l=0}^{r} (\hat{A}^{l})_{uv}. \\
        & \leq (2\alpha\beta)^{r}\cdot \frac{\sqrt{d_ud_v}}{2}\cdot\left(\frac{1}{d_u}\sum_{l=0}^{\infty} (\hat{A}^{l})_{uu} + \frac{1}{d_v}\sum_{l=0}^{\infty} (\hat{A}^{l})_{vv}- \frac{2}{\sqrt{d_ud_v}}\sum_{l=r+1}^{\infty} (\hat{A}^{l})_{uv} - R_{u,v}\right) \\
        & \leq (2\alpha\beta)^{r}\cdot \frac{\dmax}{2}\cdot\left(\frac{1}{d_u}\sum_{l=0}^{\infty} (\hat{A}^{l})_{uu} + \frac{1}{d_v}\sum_{l=0}^{\infty} (\hat{A}^{l})_{vv} -\frac{2}{\sqrt{d_ud_v}}\sum_{l=r+1}^{\infty} (\hat{A}^{l})_{uv} - R_{u,v}\right)
    \end{align*}
    We now simplify some of the terms in this bound. First, we partition the sums in the right-hand side of this equation as
    \begin{align*}
        & \frac{1}{d_u}\sum_{l=0}^{\infty} (\hat{A}^{l})_{uu} + \frac{1}{d_v}\sum_{l=0}^{\infty} (\hat{A}^{l})_{vv} - \frac{2}{\sqrt{d_ud_v}}\sum_{l=r+1}^{\infty} (\hat{A}^{l})_{uv} \\
        =& \left(\frac{1}{d_u}\sum_{l=0}^{r} (\hat{A}^{l})_{uu} + \frac{1}{d_v}\sum_{l=0}^{r} (\hat{A}^{l})_{vv}\right) \\
        &~~+ \left( \frac{1}{d_u}\sum_{l=r+1}^{\infty} (\hat{A}^{l})_{uu} + \frac{1}{d_v}\sum_{l=r+1}^{\infty} (\hat{A}^{l})_{vv} - \frac{2}{\sqrt{d_ud_v}}\sum_{l=r+1}^{\infty} (\hat{A}^{l})_{uv} \right) \\
        =& \left(\frac{1}{d_u}\sum_{l=0}^{r} (\hat{A}^{l})_{uu} + \frac{1}{d_v}\sum_{l=0}^{r} (\hat{A}^{l})_{vv}\right) \\
        &~~+ \left(\frac{1}{\sqrt{d_u}} 1_u - \frac{1}{\sqrt{d_v}} 1_v\right)^{T}\sum_{l=r+1}^{\infty} \hat{A}^{l} \left(\frac{1}{\sqrt{d_u}} 1_u - \frac{1}{\sqrt{d_v}} 1_v\right)
    \end{align*}

Let $\mu = \max\{|\mu_2|,|\mu_n|\}$. We can bound the second term in the above equation using the \textit{\textbf{Courant-Fischer Theorem}}, which says for a symmetric matrix $B$ with maximum eigenvalue $\lambda_{\max}$ and any vector $x$, one has that $x^{T}Bx \leq x^{T}x\cdot\lambda_{\max}$. Then, we have that

    \begin{align*}
         &\left(\frac{1}{\sqrt{d_u}} 1_u - \frac{1}{\sqrt{d_v}} 1_v\right)^{T}\sum_{l=r+1}^{\infty} \hat{A}^{l} \left(\frac{1}{\sqrt{d_u}} 1_u - \frac{1}{\sqrt{d_v}} 1_v\right) \\
         &~~~~\leq ~\left(\frac{1}{d_u} + \frac{1}{d_v}\right)\sum_{l=r+1}^{\infty}\mu^{l}
        ~~\leq ~\mu^{r+1}\left(\frac{1}{d_u} + \frac{1}{d_v}\right) \sum_{l=0}^{\infty} \mu^{l} \\
        &~~~~~\leq ~\mu^{r+1}\left(\frac{1}{d_u} + \frac{1}{d_v}\right) \frac{1}{1-\mu} \tag{as $\mu\in(-1,1)$} \\
        &~~~~~\leq ~\mu^{r+1}\frac{2}{\dmin} \frac{1}{1-\mu}\\
    \end{align*}

    We now bound the first term. Again, we rely on the Courant-Fischer theorem, and note that $\hat{A}^{l}_{uu} = 1_{u}^{T} \hat{A}^{l} 1_{u}$; however, as $1_{u}^{T}z_1\neq 0$, we only get a bound of $\hat{A}^{l}_{uu} \leq 1\cdot 1_{u}^{T}1_{u} = 1$. Thus,
    $$\frac{1}{d_u}\sum_{l=0}^{r} (\hat{A}^{l})_{uu} + \frac{1}{d_v}\sum_{l=0}^{r} (\hat{A}^{l})_{vv} \leq \frac{2}{\dmin} (r+1).$$
\end{proof}

%%%%%%%%%
\subsection{Proof of \Cref{thm:change_in_total_resistance}}
\label{apx:change_in_total_resistance}

In this section, we prove \Cref{thm:change_in_total_resistance}, which gives a formula for how much the effective resistance changes when an edge is added.
Recall that our strategy is to apply Woodbury's formula to compute $(L+\frac{11^{T}}{n}+(1_u-1_v)(1_u-1_v)^{T})^{-1}$.
Before doing this, we provide a proof for \Cref{lm:modified formula}.

\modifiedformula*

\begin{proof}
By Equation (7) of \cite{ghosh2008minimizing}, one has that
\[L^{+} = \left(L+\frac{11^T}{n}\right)^{-1}-\frac{11^T}{n}.\]
Then, we have that
\[(L^{+})^2 = \left(L+\frac{11^T}{n}\right)^{-2}-\frac{11^T}{n}\left(L+\frac{11^T}{n}\right)^{-1}-\left(L+\frac{11^T}{n}\right)^{-1}\frac{11^T}{n}+\frac{11^T}{n}.\]
Note that vectors of the form $1_u-1_v$ are orthogonal to the all-ones vector $1$, i.e., $(1_u-1_v)^T1=1^T(1_u-1_v)=0$. Hence
\[R_{u,v}= (1_u - 1_v)^T L^+ (1_u - 1_v)= (1_u - 1_v)^T (L+\frac{11^{T}}{n})^{-1} (1_u - 1_v),\]
and
\[B^{2}_{u,v}= (1_u - 1_v)^T (L^+)^2 (1_u - 1_v)= (1_u - 1_v)^T (L+\frac{11^{T}}{n})^{-2} (1_u - 1_v).\]

Now, by \Cref{lm: Rtot}, one has that
\[\Rtot = n\cdot\tr L^{+} =n\cdot\tr\left(\left(L+\frac{11^T}{n}\right)^{-1}-\frac{11^T}{n}\right)=n\cdot\tr\left(L+\frac{11^T}{n}\right)^{-1}-n.\]
\end{proof}

\changeintotalresistance*

\begin{proof}[Proof of \Cref{thm:change_in_total_resistance}]
Adding the edge $\{u,v\}$ to $G$ changes the Laplacian from $L$ to $L+(1_u-1_v)(1_u-1_v)^{T}$. Then, by \Cref{lm:modified formula}, we can find the difference in the total resistance by considering difference of $\Rtot(G) = n\cdot\tr(L+\frac{11^{T}}{n})^{-1} -n$ and $\Rtot(G\cup\{u,v\}) = n\cdot\tr(L+\frac{11^{T}}{n}+(1_u-1_v)(1_u-1_v)^{T})^{-1}-n.$ The difference of these is the trace of the third term in Woodbury's formula, which simplifies to the quantity in the statement as follows.
    \begin{align*}
        & \Rtot(G) - \Rtot(G\cup\{u,v\}) \\
        =& n\cdot\tr\left(L+\frac{11^{T}}{n}\right)^{-1} - n\cdot\tr\left(L+\frac{11^{T}}{n}+(1_u-1_v)(1_u-1_v)^{T}\right)^{-1} \\
        = & n\cdot\tr\left(\left(1 + (1_u-1_v)^{T}\left(L+\frac{11^{T}}{n}\right)^{-1}(1_u-1_v)\right)^{-1}\cdot \left(\left(L+\frac{11^{T}}{n}\right)^{-1}(1_u-1_v)\right)\left(\left(L+\frac{11^{T}}{n}\right)^{-1}(1_u-1_v)\right)^{T}  \right) \\
        = & n\cdot\underbrace{\left(1 + (1_u-1_v)^{T}\left(L+\frac{11^{T}}{n}\right)^{-1}(1_u-1_v)\right)^{-1}}_{c}\cdot \tr \left(\left(\left(L+\frac{11^{T}}{n}\right)^{-1}(1_u-1_v)\right)\left(\left(L+\frac{11^{T}}{n}\right)^{-1}(1_u-1_v)\right)^{T} \right).
    \end{align*}
    For the coefficient term $c$, one has that
     $$
    \left(1 + (1_u-1_v)^{T}\left(L+\frac{11^{T}}{n}\right)^{-1}(1_u-1_v)\right) = (1+R_{u,v}).
     $$
    For the trace term, one has that
    \begin{align*}
        & \tr\left(\left(\left(L+\frac{11^{T}}{n}\right)^{-1}(1_u-1_v)\right)\left(\left(L+\frac{11^{T}}{n}\right)^{-1}(1_u-1_v)\right)^{T}\right) \\
        &= (1_u-1_v)\left(L+\frac{11^{T}}{n}\right)^{-2}(1_u-1_v)^{T} \\
        &= B^{2}_{u,v}
    \end{align*}
    by the fact that $\tr(xx^{T}) = x^{T}x$ for any vector $x$.
\end{proof}

\section{Runtime Analysis of GTR}
\label{apx:runtime}

\subsection{Asymptotic Analysis}

The time complexity for GTR rewiring depends on the time it takes to (step 1) compute the effective resistance and biharmonic distance for each pair of vertices, (step 2) find the pair of vertices maximizing $R_{u,v}/(1+B^{2}_{u,v})$, and (step 3) update the effective resistance and biharmonic distance.
%\par
If we are adding $k$ edges to the graph, the naive implementation for GTR takes $O(n^{3}+kn^{2})$ time. We can compute $L^{+}$ and $L^{2+}$ in $O(n^{3})$ time using the singular value decomposition, which we can use to compute all pairs effective resistance and biharmonic distance in time $O(n^2)$. In total, step (1) would take $O(n^{3} + n^{2})$ time. Step (2) would take $O(n^2)$ time to iterate over all pairs of vertices. Finally, for step (3), we can update $L^{+}$ and $L^{2+}$ in $O(n^2)$ time. This is because adding the edge $\{u,v\}$ to $G$ only causes a constant-rank change to the Laplacian; the Laplacian changes from $L$ to $L+(1_u-1_v)(1_u-1_v)^{T}$ and the squared Laplacian changes from $L^{2}$ to $L^{2}+L(1_u-1_v)(1_u-1_v)^{T} + (1_u-1_v)(1_u-1_v)^{T}L + (1_u-1_v)(1_u-1_v)^{T}$. The pseudoinverse of $L^{+}$ and $L^{2+}$ can then be updated in $O(n^{2})$ using Woodbury's Formula (see \Cref{lm:woodbury}).
\footnote{In general, Woodbury's Formula cannot be used to update the pseudoinverse of a matrix; however, in the special case of adding an edge to a connected graph, it can be used to update the pseudoinverse of $L$ and $L^{2}$. In short, this is because the vector $1_u-1_v$ is orthogonal to the kernels of $L$ and $L^{2}$. See the discussion in \Cref{apx:change_in_total_resistance}.}
% \par

However, more efficient implementations for GTR are possible thanks to nearly-linear time Laplacian solvers: algorithms for solving linear systems of the form $Lx=b$ in $O(m\poly\log n)$ time~\cite{spielman2004solver,jambulapati2021ultrasparse}. Using these algorithms, the pseudoinverses $L^{+}$ and $L^{2+}$ could be computed in $O(n\cdot m\poly\log n)$ by using Laplacian solvers to find the columns of the matrix. Alternatively, all-pairs effective resistance and biharmonic distance can be estimated in time $O(m\poly\log n + n^2\poly\log n)$ using an algorithm that combines Laplacian solvers and Johnson-Lindenstrauss random projection~\cite{spielman2011sparsifiers}.

\subsection{Experimental Analysis}

We implemented the GTR algorithm in PyTorch Geometric; our analysis is available here:  \url{https://github.com/blackmit/gtr_rewiring}. The fastest implementation of GTR we found was to use the naive algorithm; this is because we can calculate the pseudoinverse of the Laplacian using a GPU. The following table contains the amount of time needed to compute 50 edges using each algorithm.

\begin{table}[H]
    \centering
    \begin{tabular}{c|cccccc}
         & MUTAG &  PROTEINS & ENZYMES & IMDB-BINARY & REDDIT-BINARY & COLLAB \\
         \hline
        FoSR & 0.10 & 1.00 & 0.37 & 0.51 & 199.20 & 15.94  \\
        GTR & 12.86 & 68.10 & 35.76 & 57.23 & 349.98 & 423.79
    \end{tabular}
    \caption{Time in seconds to compute 50 additional edges to add to the graph using both FoSR and GTR}
    \label{tab:compute_times}
\end{table}

\section{Counterexamples to the Optimality of GTR.}\label{sec:counterexample}

\Cref{thm:change_in_total_resistance} proves that GTR adds the single edge that most decreases the total resistance; however, GTR will not necessarily add the $k$ edges that most decrease total resistance for $k>1$. \Cref{fig:gtr_optimality_counterexample} gives an example where this is the case.

\begin{figure}[ht]
    \centering
    \begin{subfigure}{0.3\textwidth}
        \centering
        \includegraphics[width=1.5in]{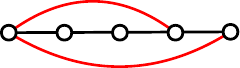}
        \caption*{GTR}
    \end{subfigure}
    \begin{subfigure}{0.3\textwidth}
        \centering
        \includegraphics[width=1.5in]{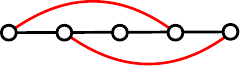}
        \caption*{Optimal}
    \end{subfigure}
    \caption{The path on 5 vertices is a counterexample showing that GTR does not add the $k$ edges that most decrease $\Rtot$ when $k>1$. Left: The two edges added by GTR. GTR first adds the edge connecting the first and last vertex in the path. The total resistance of this graph is $\Rtot\approx8.18$. Right: The two edges that most deceases the total resistance. The total resistance of this graph is $\Rtot\approx 7.67$.}
    \label{fig:gtr_optimality_counterexample}
\end{figure}

The amount an edge decreases the total resistance can increase as more edges are added to the graph. \Cref{fig:monotoncity_counterexample} gives such an example. This can be interpreted as an edge becoming more important for the global topology of the graph as the graph changes.

\begin{figure*}[ht]
    \centering
    \begin{subfigure}{1.0\textwidth}
        \centering
        \includegraphics[width=6.5in]{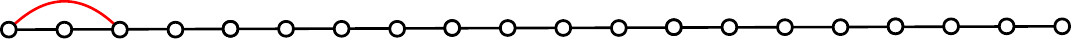}
    \end{subfigure}
    \begin{subfigure}{1.0\textwidth}
        \centering
        \vspace{0.25in}
    \end{subfigure}
    \begin{subfigure}{1.0\textwidth}
        \centering
        \includegraphics[width=6.5in]{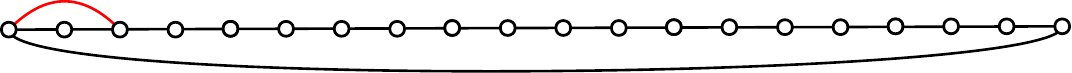}
    \end{subfigure}
    \caption{The path on 20 vertices is an example showing that the amount an edge decreases the total resistance is not monotonic. Top: Adding the red edge would decrease the total resistance by $\approx 30.33$. Bottom: After adding the edge connecting the first and last vertex in the path, adding the red edge would decrease the total resistance by $\approx 40.17$}
    \label{fig:monotoncity_counterexample}
\end{figure*}

\section{Experimental Details}
\label{sec:experimental_details}

We use the same configuration of hyperparameters as in~\cite{karhadkar2022firstorder}. We use randomly generated $80\%/10\%/10\%$ train/validation/test splits of the data. We use the Adam optimizer and the \texttt{ReduceLROnPlateau} scheduler in Torch that reduces the learning rate after 10 epochs without an improvement in the validation accuracy. We use a stopping patience of 100 epochs of the validation loss. For the hyperparameter search, we consider average accuracies over 10 randomly generated splits of the data. For the test results, we report the average test accuracy and $95\%$ confidence intervals over 100 randomly generated splits.

\begin{table}[htb!]
\centering
\caption{Number of edges added by GTR of FoSR for each dataset. Note that FoSR only contains the number of edges when our run in the edge ablation experiment (\Cref{sec:edge_ablation}) outperformed the run in \cite{karhadkar2022firstorder}. The number of edges added by FoSR for all other experiments can be found in the appendix to \cite{karhadkar2022firstorder}.}
\begin{tabular}{c c c c c c c}
    \hline
    &&& GCN &&& \\
    \hline
    Rewiring & Mutag & Proteins & Enzymes & Reddit-Binary & IMDB-Binary & Collab \\
    \hline
    GTR & 45 & 25  & 20 & 5 & 5 & 5 \\
    \hline
    &&& R-GCN &&& \\
    \hline
    Rewiring & Mutag & Proteins & Enzymes & Reddit-Binary & IMDB-Binary & Collab \\
    \hline
    GTR & 50 & 10 & 40 & 20 & 40 & 25 \\
    \hline
    &&& GIN &&& \\
    \hline
    Rewiring & Mutag & Proteins & Enzymes & Reddit-Binary & IMDB-Binary & Collab \\
    \hline
    GTR & 25 & 5 & 5 & 5 & 15 & 25 \\
    \hline
    &&& R-GIN &&& \\
    \hline
    Rewiring & Mutag & Proteins & Enzymes & Reddit-Binary & IMDB-Binary & Collab \\
    \hline
    FoSR & - & 20 & - & 25 & 50 & 20 \\
    GTR & 15 & 5  & 50 & 5 & 20 & 30 \\
    \hline
\end{tabular}
\label{table:gtr_number_edges_added}
\end{table}

\begin{table}[htb!]
\centering
\caption{Hyperparameters for Graph Classifcation. These are consistent across all GNN types. These are the same as used in the experiments in~\cite{karhadkar2022firstorder}}
\begin{tabular}{|c | c |}
    \hline
    Hyperparameters & \\
    \hline
    Number of Hidden Layers &  4 \\
    Dimension of Hidden Layers & 64 \\
    Dropout & 0.5 \\
    Learning Rate & $1.0\times 10^{-3}$ \\
    \hline
\end{tabular}
\label{table:hyperparameters}
\end{table}

\section{Total Resistance vs.~Number of Edges Added}
\label{apx:total_resistance_curves}

\Cref{fig:resistance_curves} shows the decrease in average total resistance across a dataset as edges are added to a graph by GTR or FoSR. GTR seems to outperform FoSR in decreasing total resistance.

\section{Edge Ablation}
\label{apx:edge_ablation}

\Cref{fig:edge_ablation} shows the effect of adding between 0 and 50 edges on the classification accuracy across different graph classification datasets. We used the R-GIN architecture for the experiments and followed the same experimental procedure as described in \Cref{sec:experimental_details}.
\par
We see a variety of behaviors across the datasets. For some datasets like Proteins or IMDB-Binary, we see an initial large jump in accuracy after adding a few edges, but generally see little improvement by adding more edges. For Enzymes, the accuracy almost only increases as we add more edges, suggesting that the optimal number of edges was greater than the maximum of 50 we tested. The variety of behaviors suggest that there is no optimal number of edges to add that will maximize performance across datasets. Our experiments also suggest that, while adding some number of edges helps for all datasets, performance doesn't continue to increase as more edges are added.
\par
For almost all datasets, we see the greatest rate of improvement in accuracy after adding a few edges. A possible explanation might be that the rate total resistance decreases is greatest for the first few edges added, as we see in \Cref{fig:resistance_curves}.

\section{Hidden Dimension Ablation}
\label{apx:hidden_dim_ablation}

\Cref{fig:hidden_dim_ablation} shows the effect of adding between 0 and 30 edges and using a hidden dimension of 32, 64, or 128 on graph classification accuracy. We used the R-GIN architecture for these experiments and followed the same experimental procedure as described in \Cref{sec:experimental_details}. Generally, we see that both rewiring and increasing the hidden dimension improve the classification accuracy.

\begin{figure}[ht]
    \centering
    \begin{subfigure}{0.33\textwidth}
        \centering
        \includegraphics[width=\linewidth]{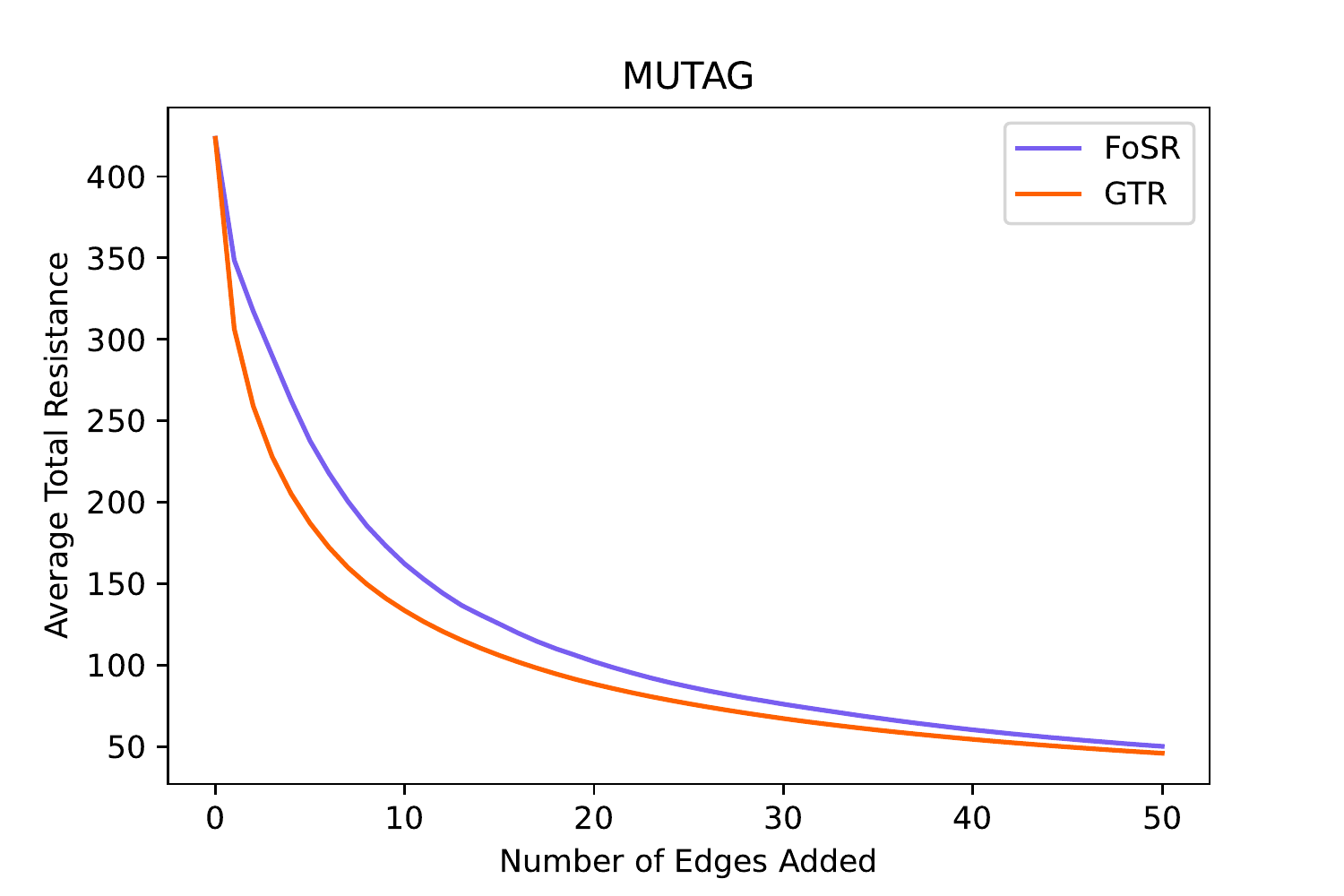}
    \end{subfigure}
    \begin{subfigure}{0.33\textwidth}
        \centering
        \includegraphics[width=\linewidth]{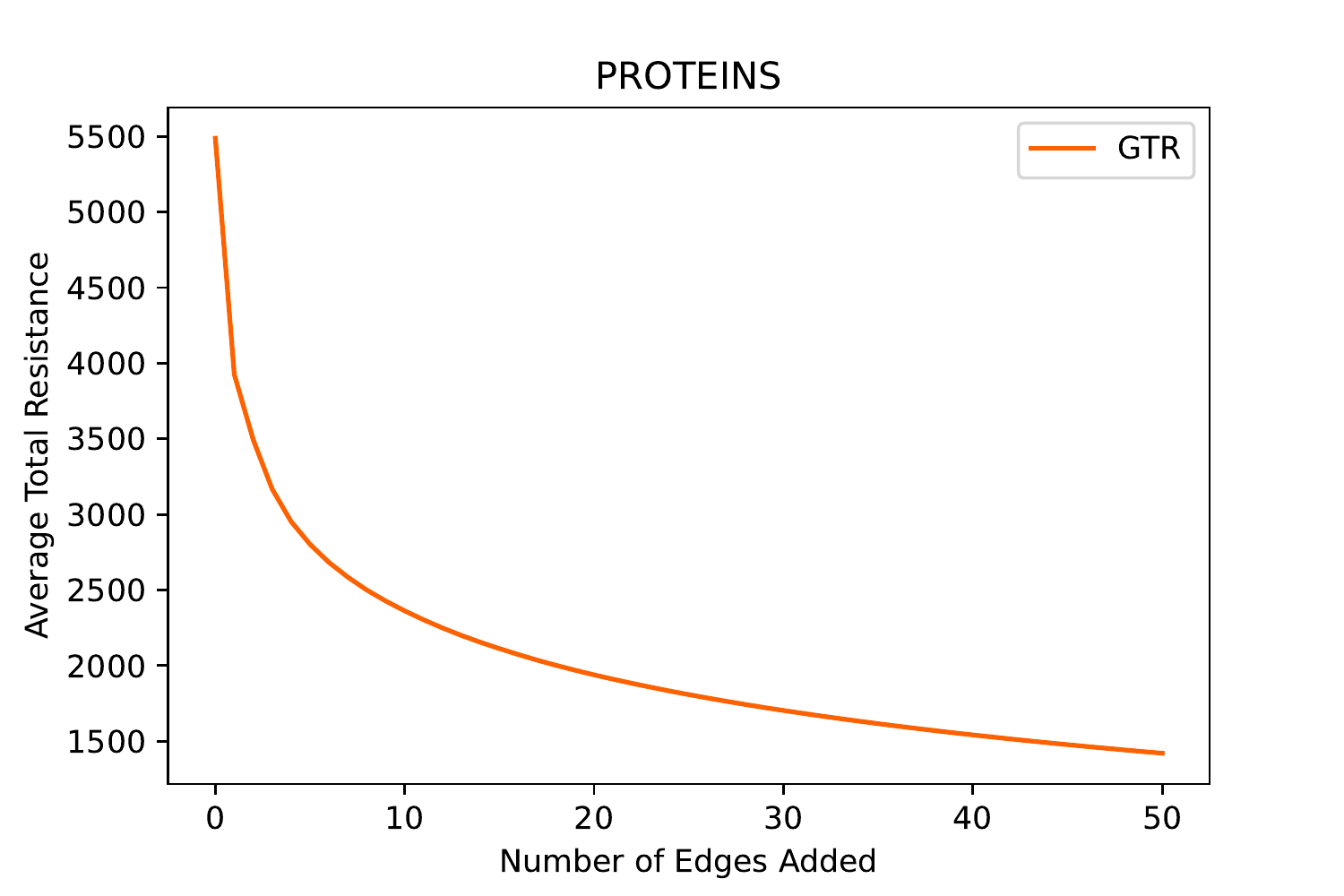}
    \end{subfigure}
    \begin{subfigure}{0.33\textwidth}
        \centering
        \includegraphics[width=\linewidth]{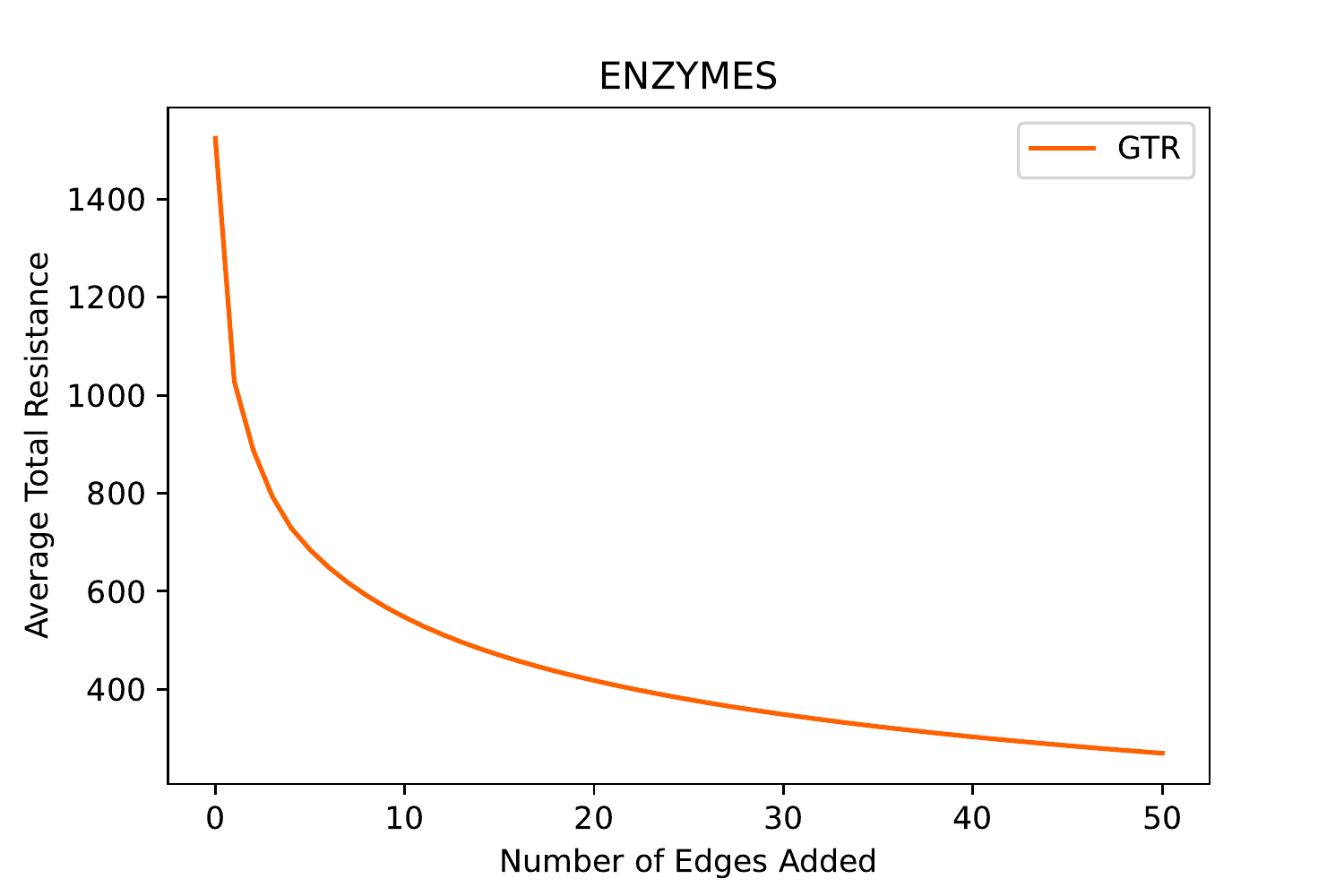}
    \end{subfigure}
    \begin{subfigure}{0.33\textwidth}
        \centering
        \includegraphics[width=\linewidth]{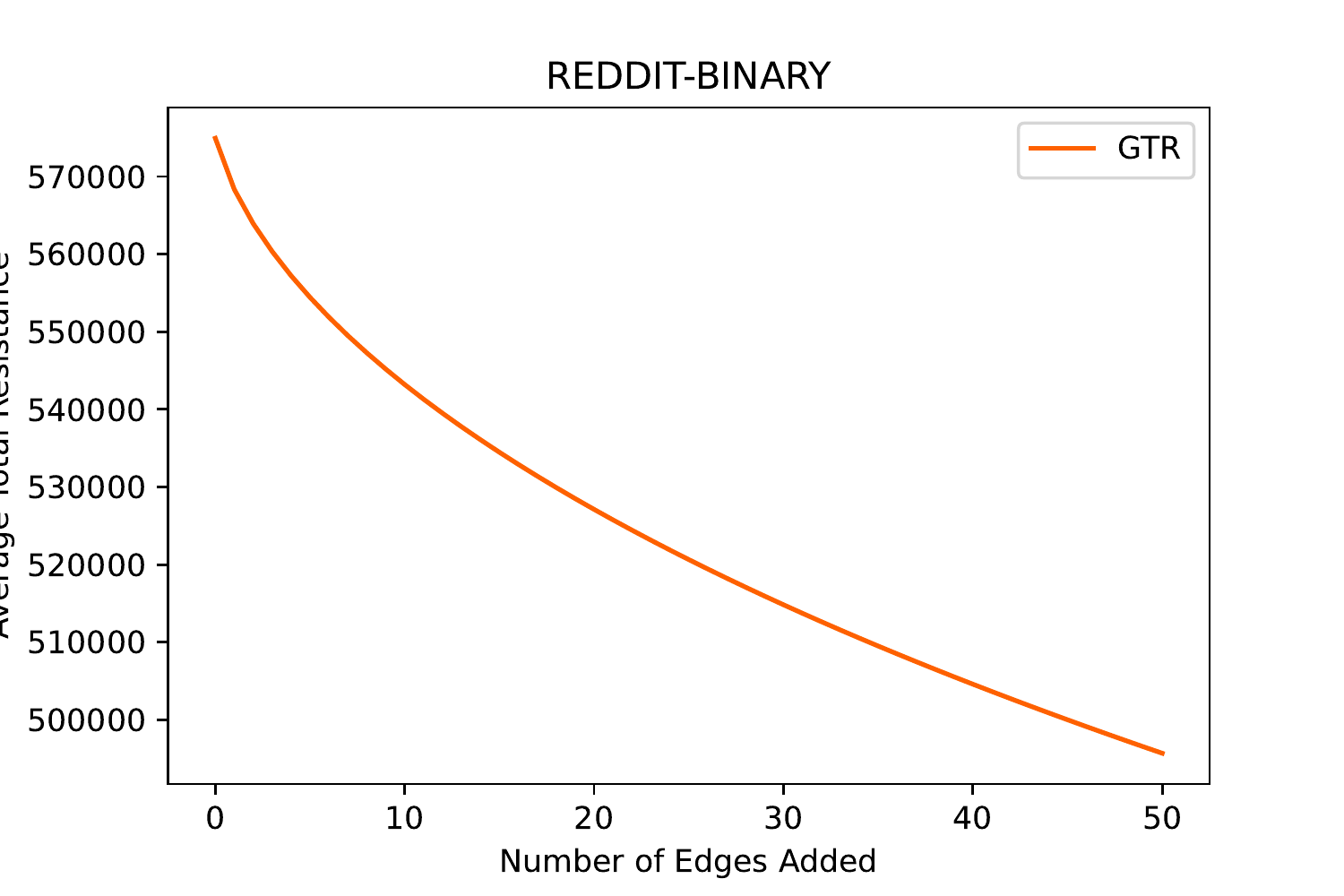}
    \end{subfigure}
    \begin{subfigure}{0.33\textwidth}
        \centering
        \includegraphics[width=\linewidth]{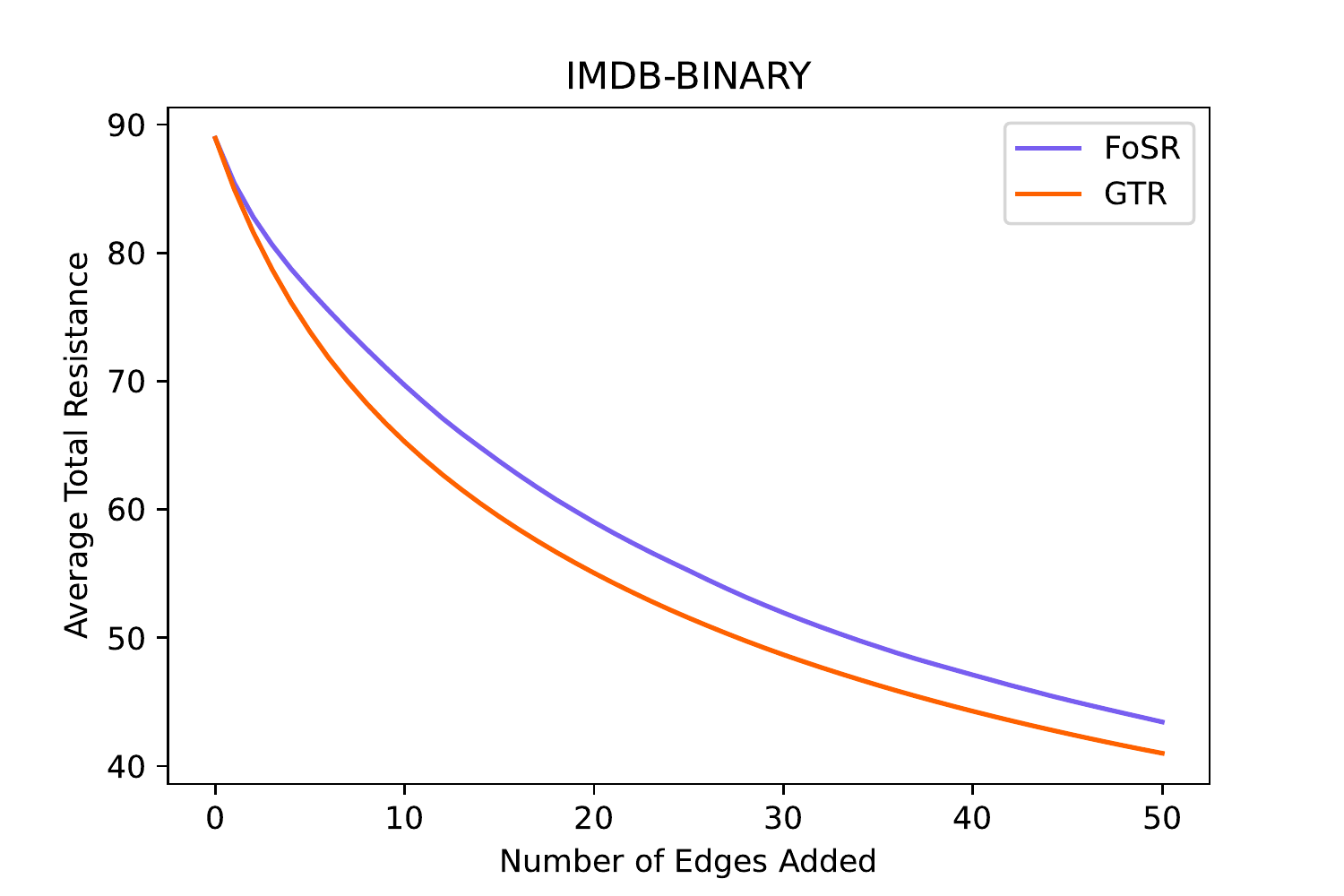}
    \end{subfigure}
    \begin{subfigure}{0.33\textwidth}
        \centering
        \includegraphics[width=\linewidth]{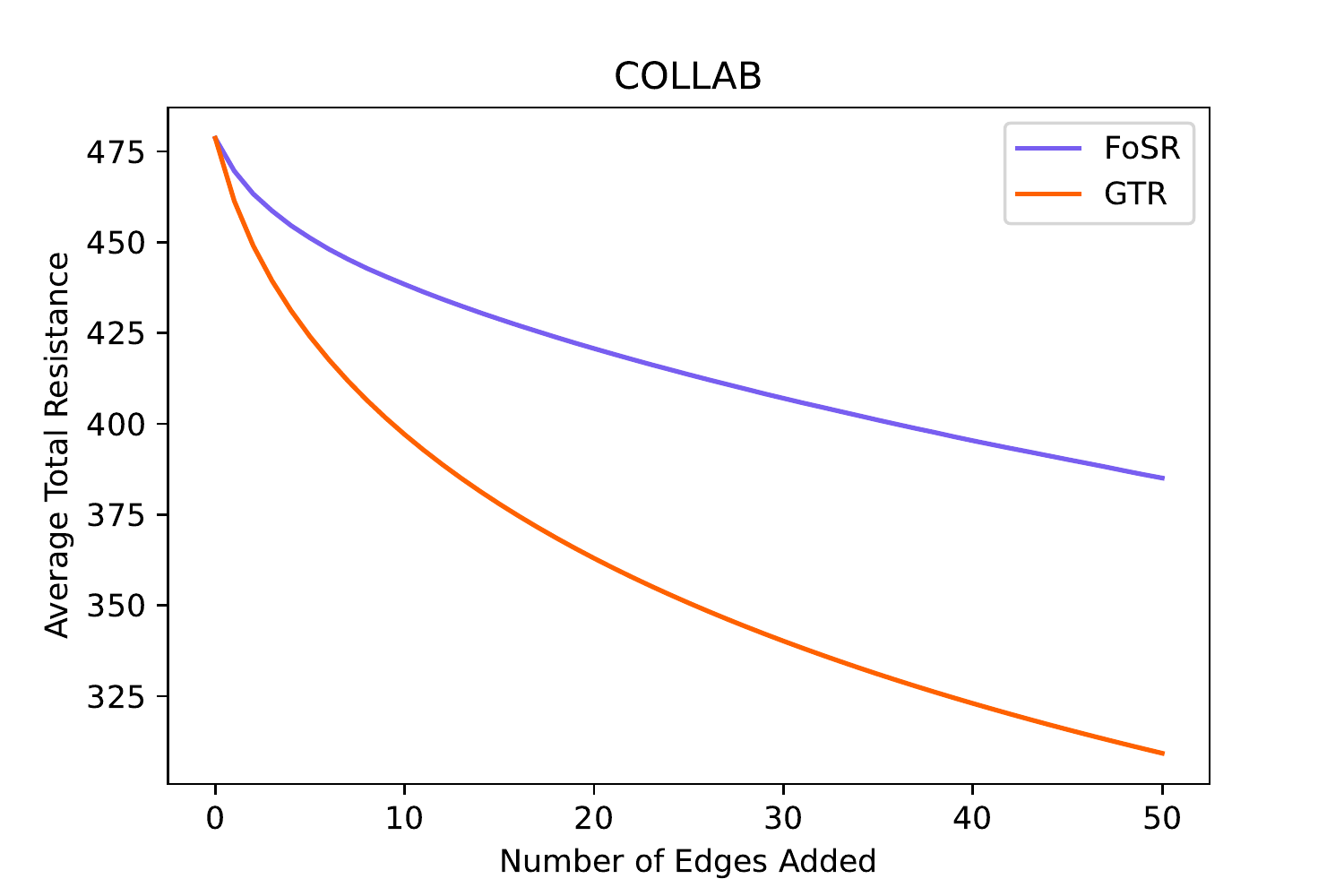}
    \end{subfigure}
    \caption{Plots of average total resistance vs.~number of edges added. For disconnected graphs, plots show the sum of effective resistances for all pairs of vertices in the same connected component, as effective resistance between vertices in different connected components is ill-defined. As FoSR adds edges between different connected components and GTR does not, it would not be meaningful to compare total effective resistance for datasets with disconnected graphs (i.e., Proteins, Enzymes, and IMDB-Binary) as FoSR may connect these disconnected componets, which is why FoSR curves are not reported for these datasets.}
    \label{fig:resistance_curves}
\end{figure}

\begin{figure}[h]
    \centering
    \begin{subfigure}{0.33\textwidth}
        \centering
        \includegraphics[width=\linewidth]{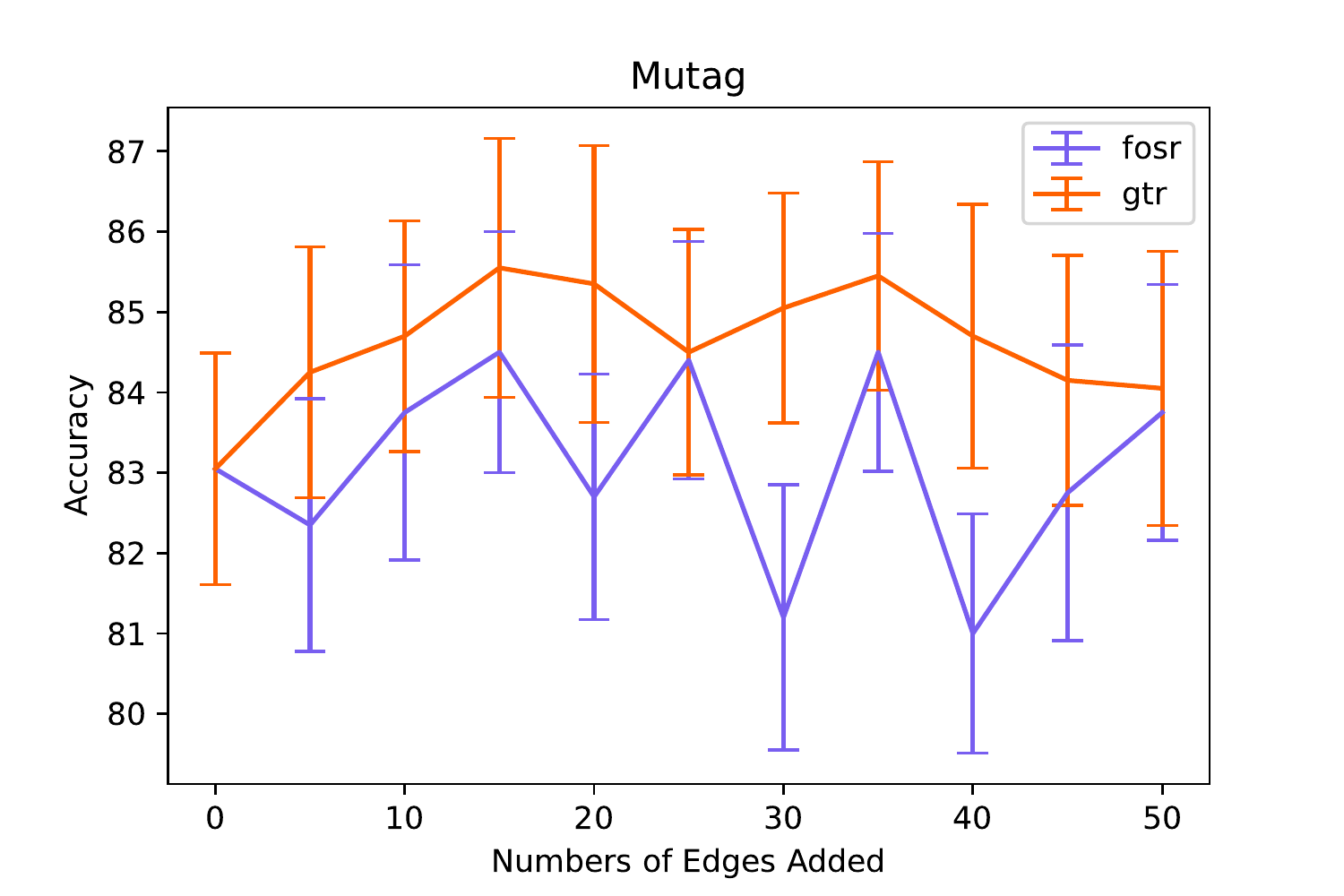}
    \end{subfigure}
    \begin{subfigure}{0.33\textwidth}
        \centering
        \includegraphics[width=\linewidth]{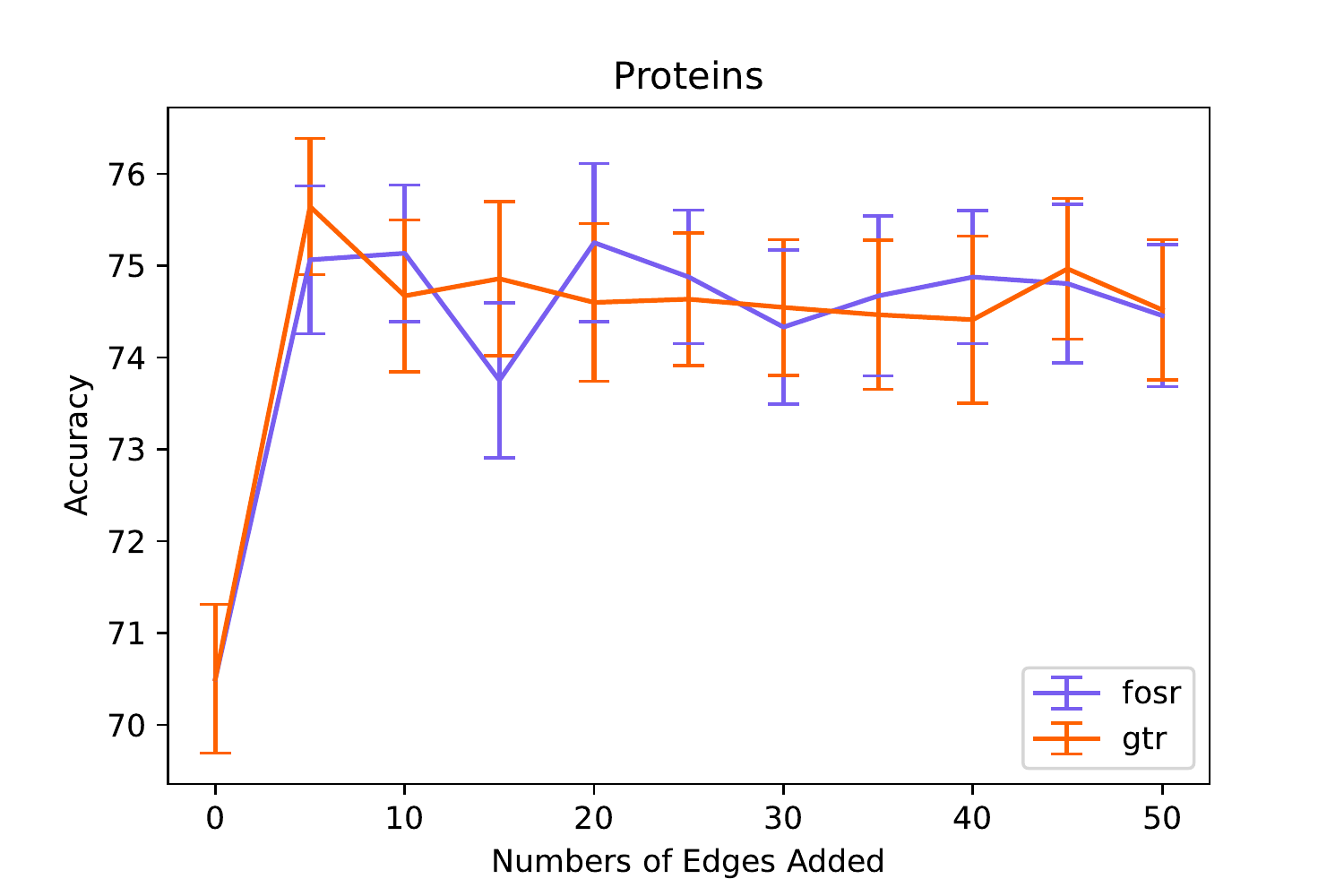}
    \end{subfigure}
    \begin{subfigure}{0.33\textwidth}
        \centering
        \includegraphics[width=\linewidth]{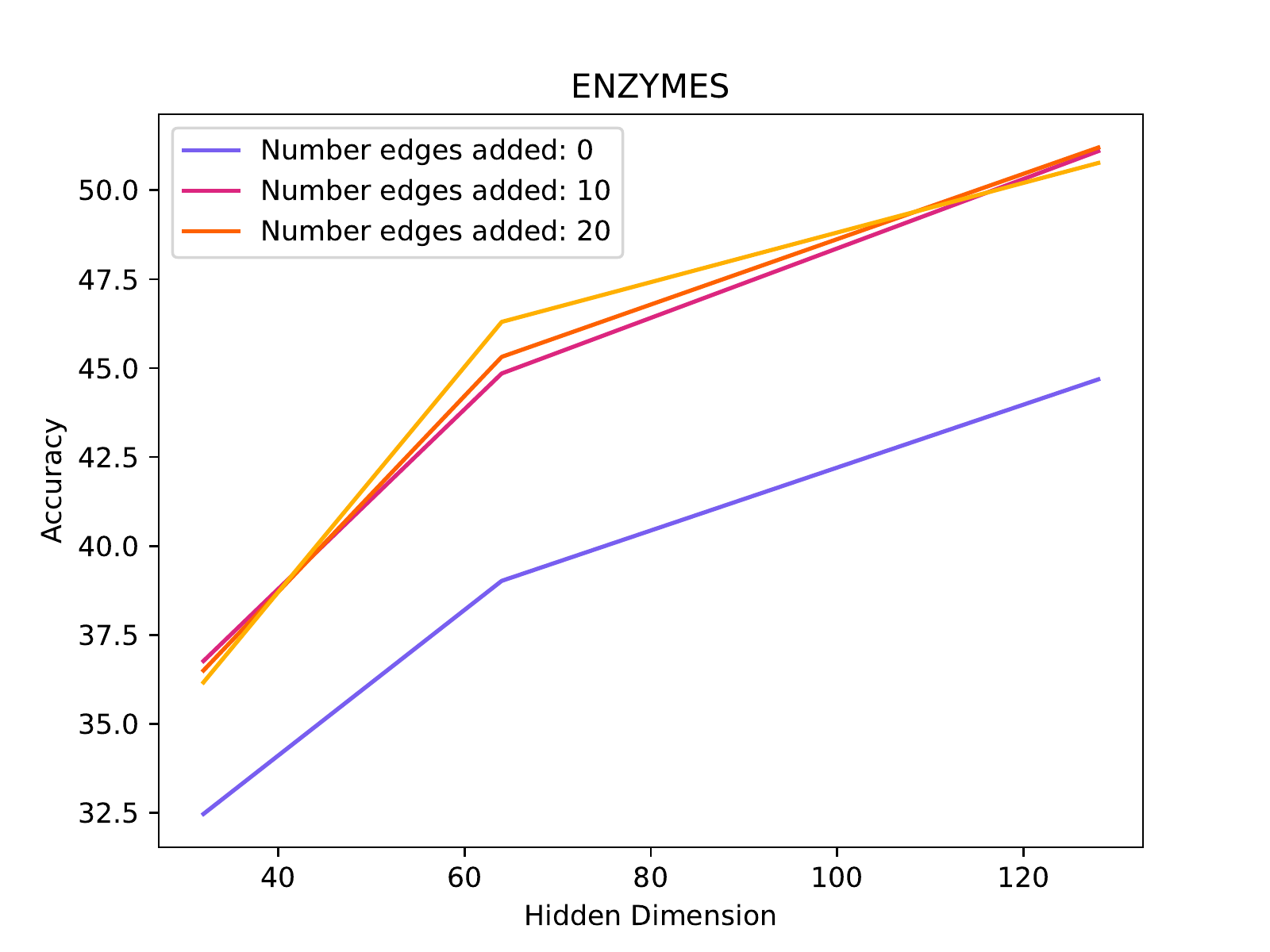}
    \end{subfigure}
    \begin{subfigure}{0.33\textwidth}
        \centering
        \includegraphics[width=\linewidth]{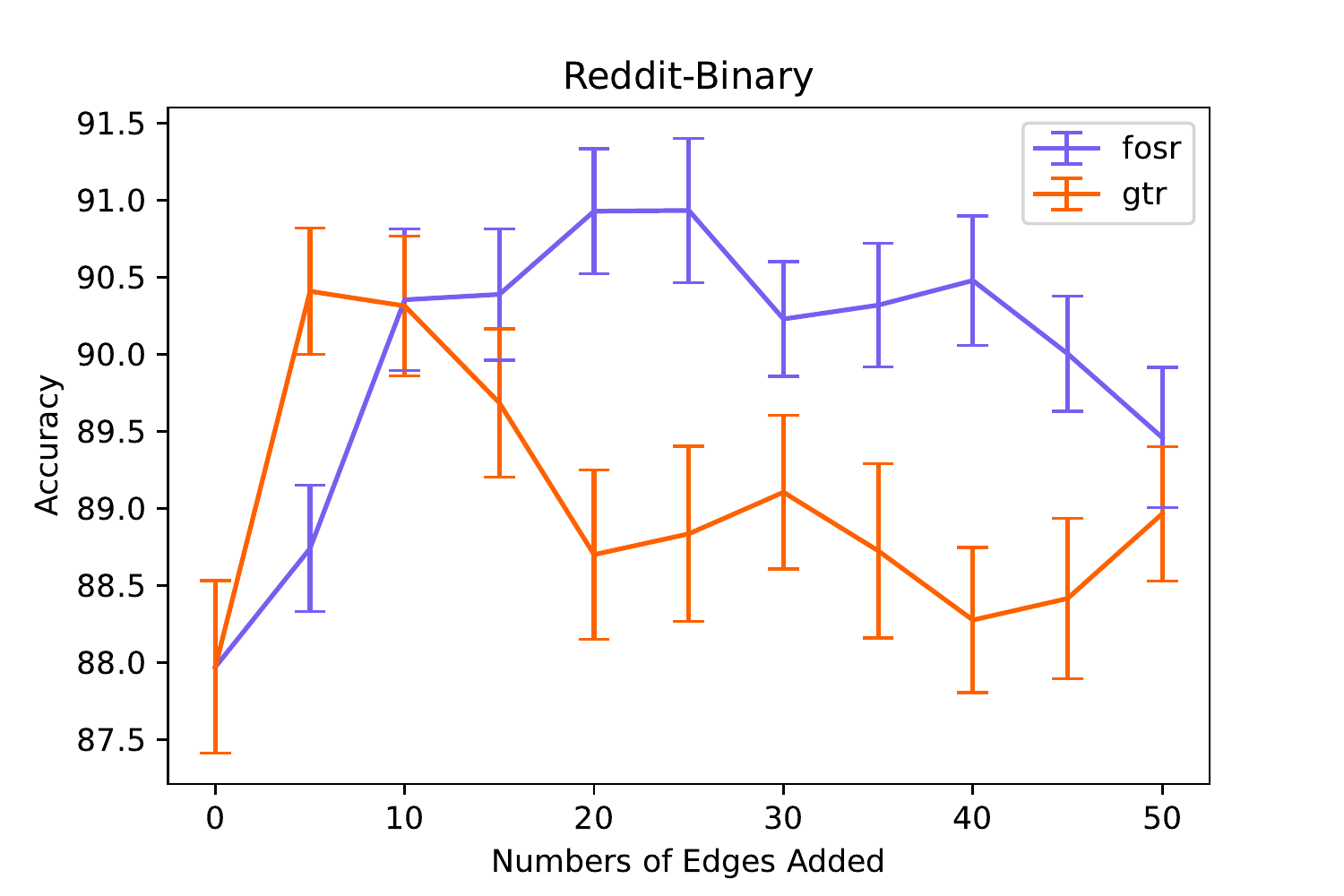}
    \end{subfigure}
    \begin{subfigure}{0.33\textwidth}
        \centering
        \includegraphics[width=\linewidth]{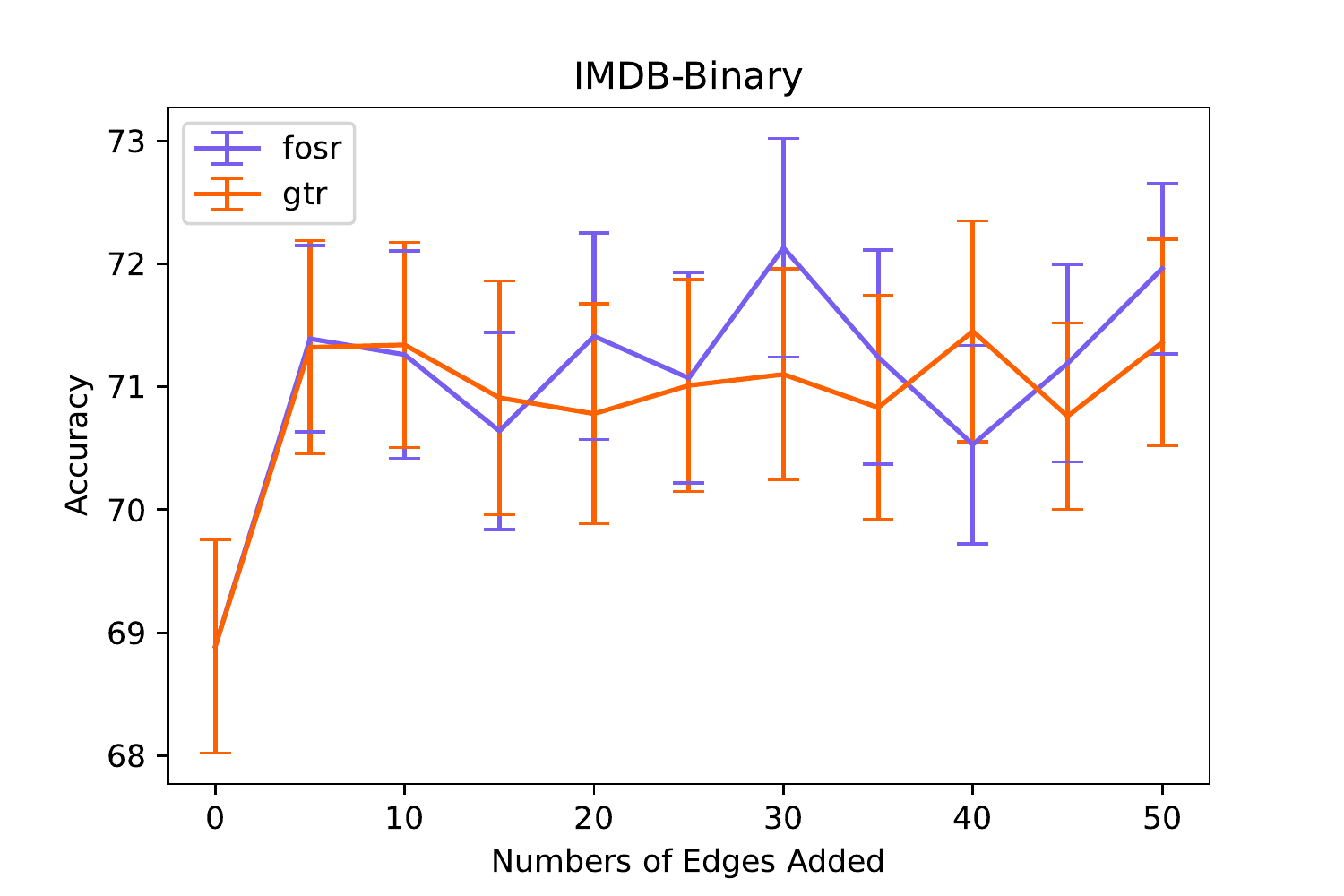}
    \end{subfigure}
    \begin{subfigure}{0.33\textwidth}
        \centering
        \includegraphics[width=\linewidth]{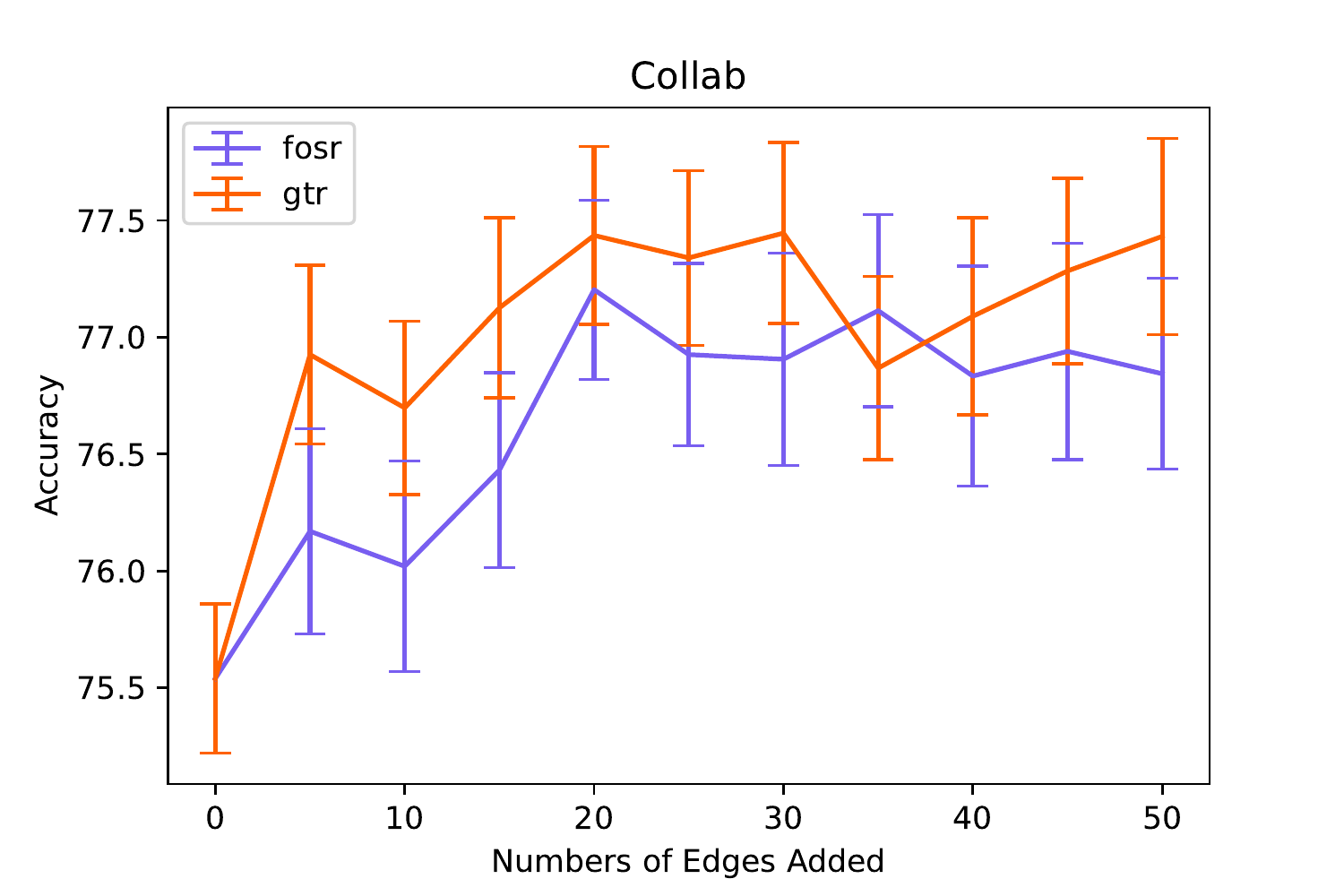}
    \end{subfigure}
    \caption{Plots of graph classification accuracy vs.~number of edges added by GTR or FoSR.}
    \label{fig:edge_ablation}
\end{figure}

\begin{figure}[h]
    \centering
    \begin{subfigure}{0.33\textwidth}
        \centering
        \includegraphics[width=\linewidth]{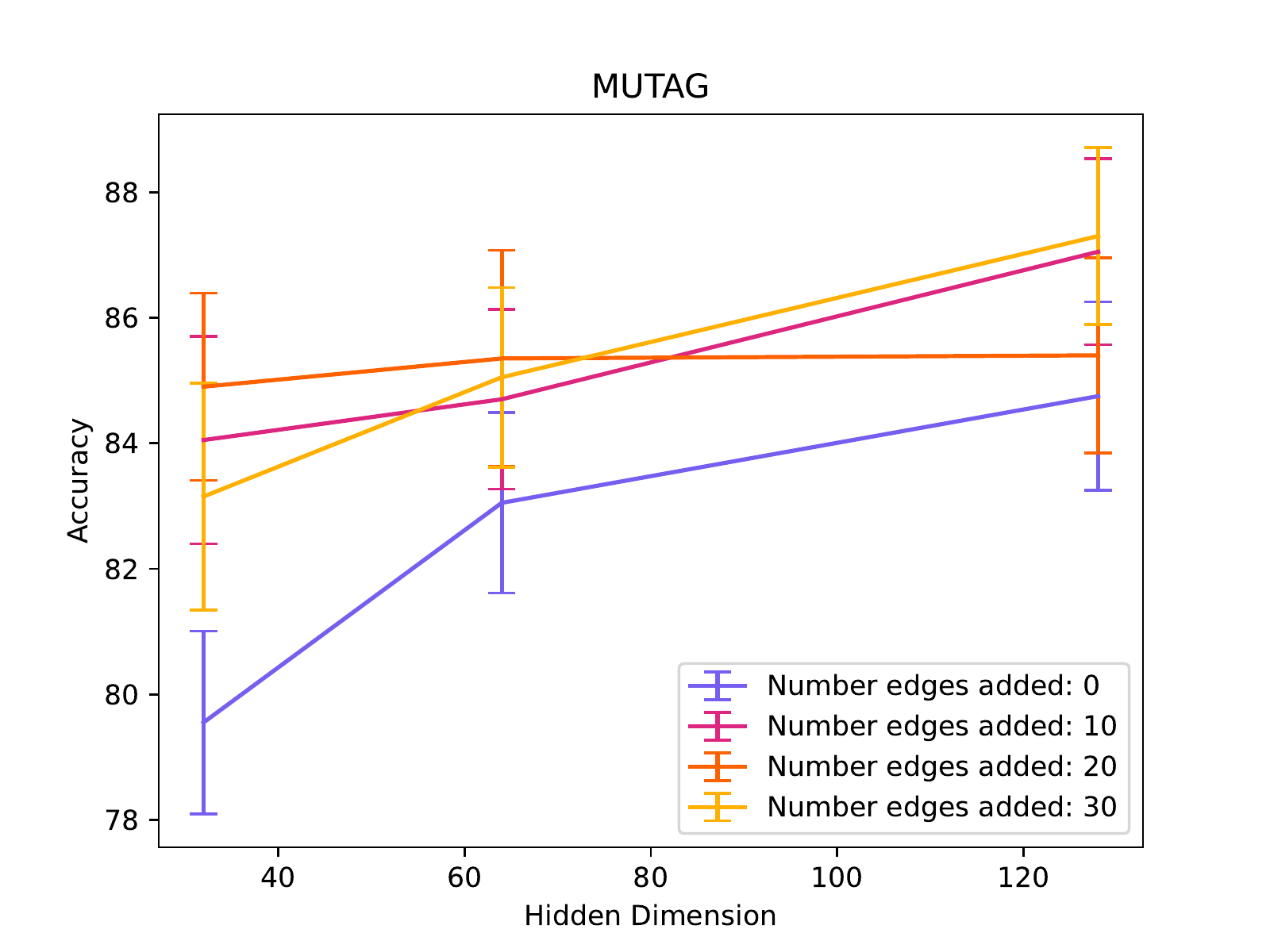}
    \end{subfigure}
    \begin{subfigure}{0.33\textwidth}
        \centering
        \includegraphics[width=\linewidth]{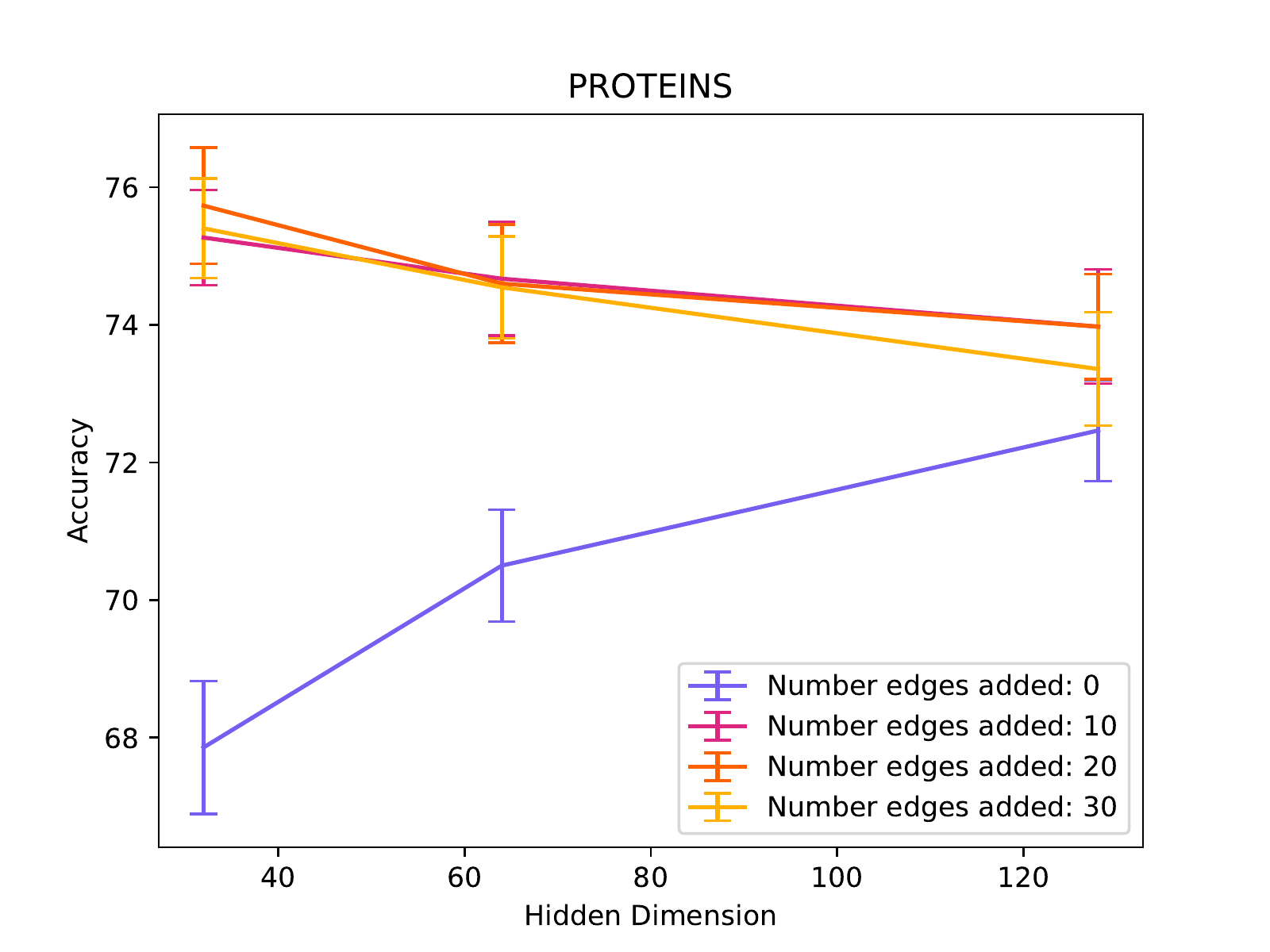}
    \end{subfigure}
    \begin{subfigure}{0.33\textwidth}
        \centering
        \includegraphics[width=\linewidth]{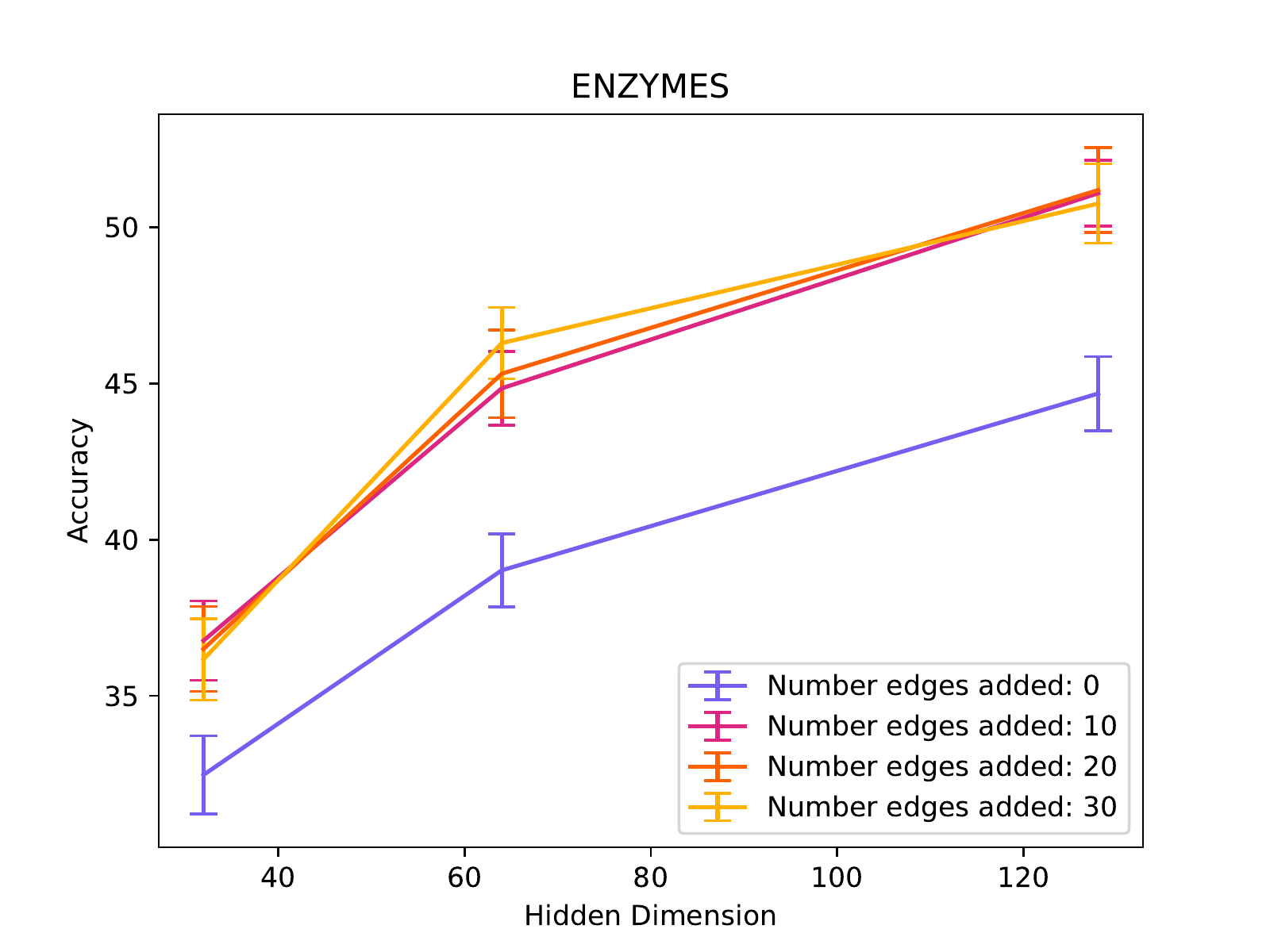}
    \end{subfigure}
    \begin{subfigure}{0.33\textwidth}
        \centering
        \includegraphics[width=\linewidth]{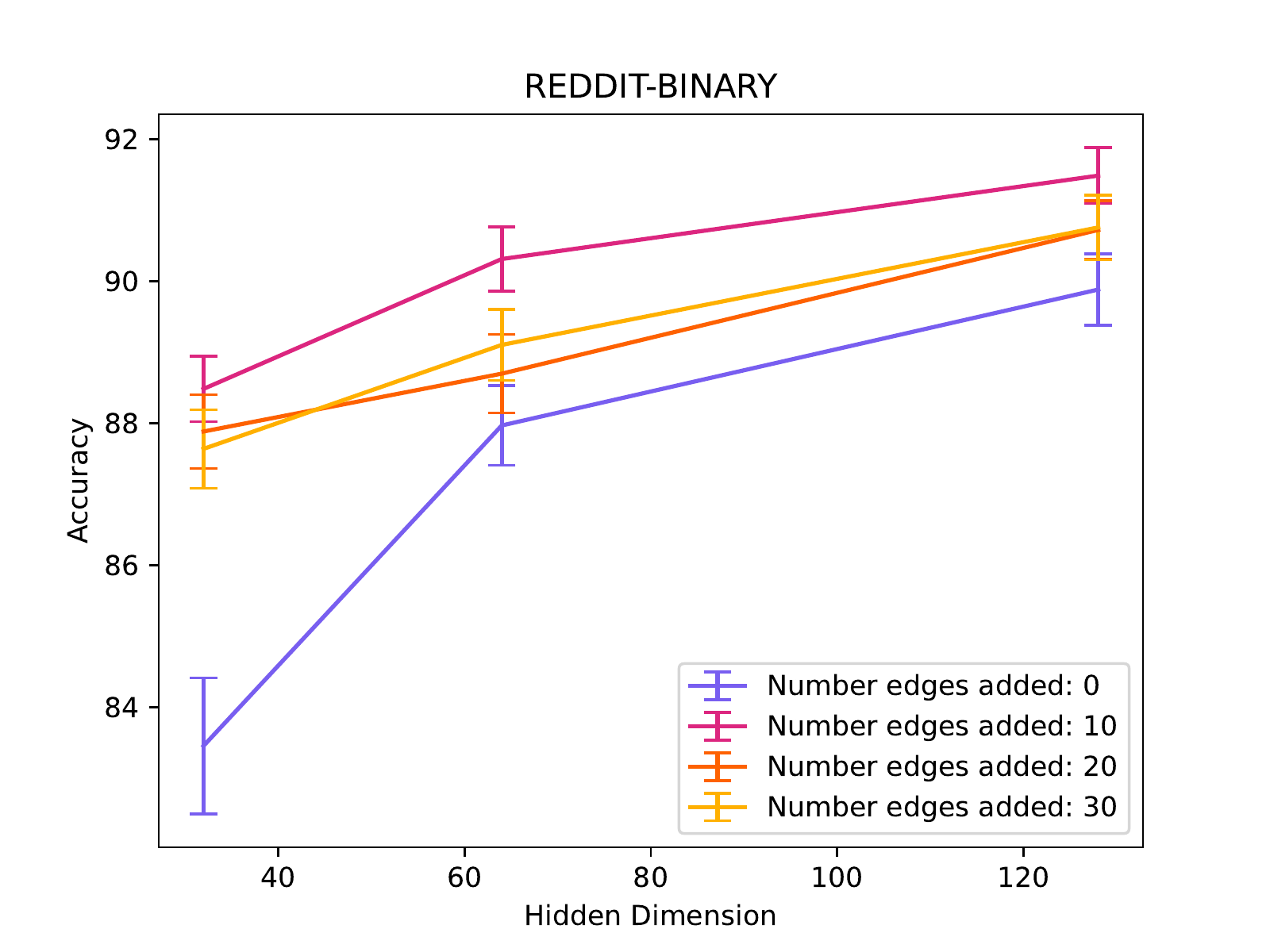}
    \end{subfigure}
    \begin{subfigure}{0.33\textwidth}
        \centering
        \includegraphics[width=\linewidth]{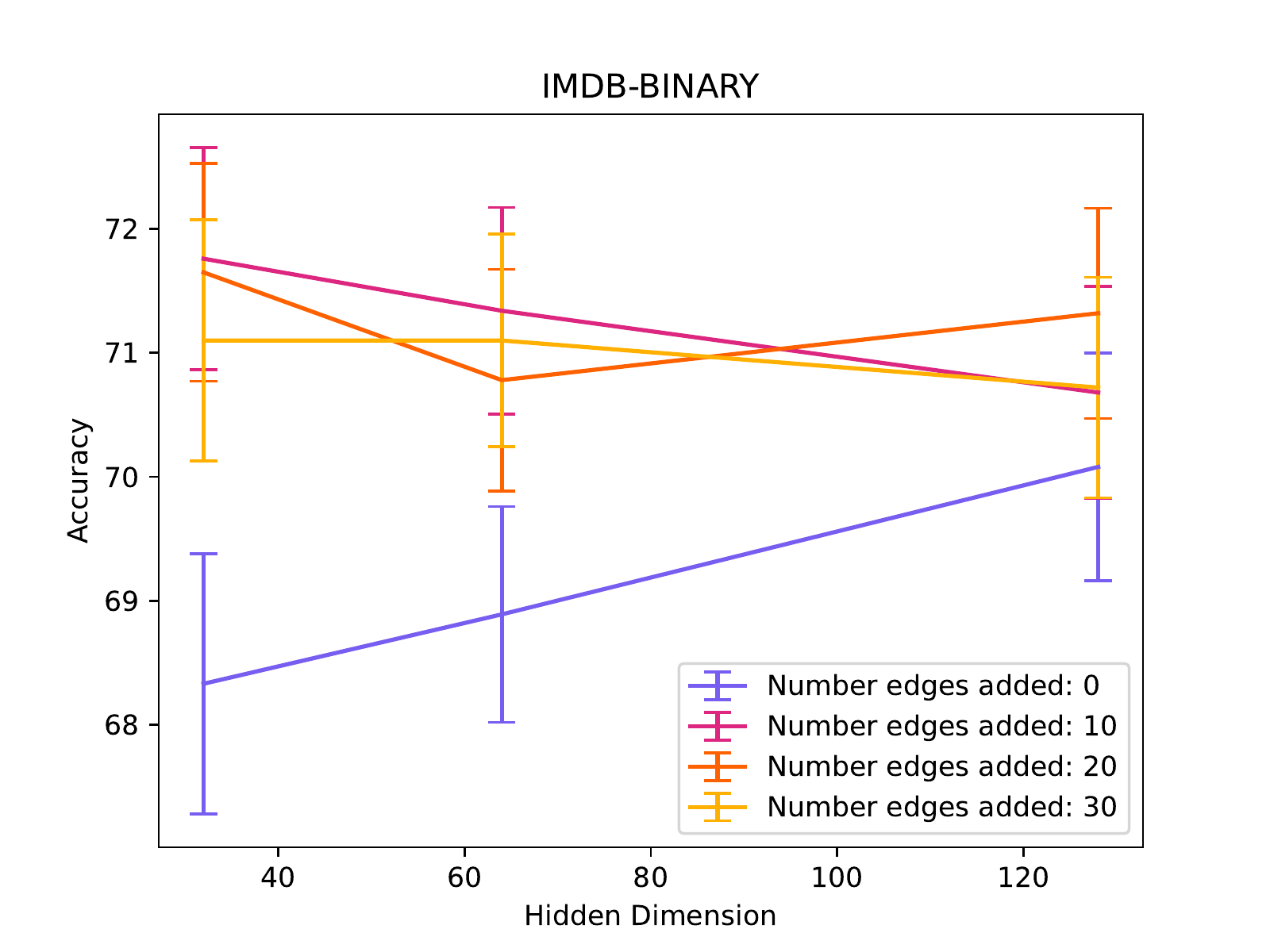}
    \end{subfigure}
    \begin{subfigure}{0.33\textwidth}
        \centering
        \includegraphics[width=\linewidth]{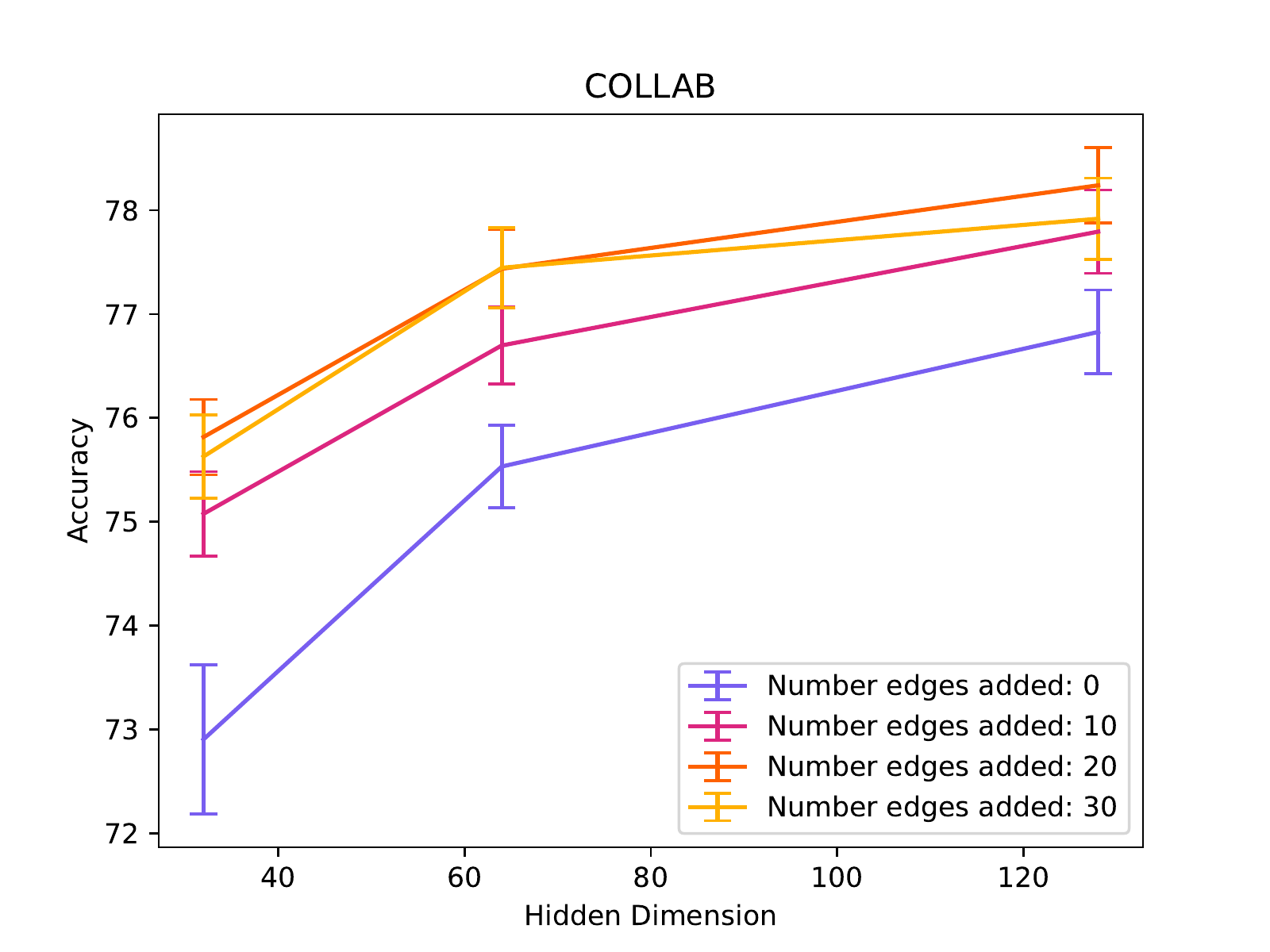}
    \end{subfigure}
    \caption{Plots of graph classification accuracy vs.~hidden dimension for a variable number of edges added by GTR.}
    \label{fig:hidden_dim_ablation}
\end{figure}

\end{document}